\documentclass[conference]{IEEEtran}
\usepackage{times}

\usepackage{graphicx}
\usepackage{amsmath}
\usepackage{caption}
\usepackage{subcaption}
\usepackage[numbers]{natbib}
\usepackage{multicol}
\usepackage[bookmarks=true]{hyperref}
\usepackage[utf8]{inputenc}

\usepackage{tikz}

\usepackage{booktabs}

\DeclareCaptionLabelSeparator{periodspace}{.\quad}
\newcommand{\usertext}{\text{usr}}
\newcommand{\robottext}{\text{rob}}
\newcommand{\policy}{\pi}
\newcommand{\policyuser}{\policy^\usertext}
\newcommand{\policyusergoal}{\policy^\usertext_\goal}

\newcommand{\goal}{g}
\newcommand{\Goal}{G}

\newcommand{\actionuser}{u}
\newcommand{\Actionuser}{U}
\newcommand{\actionrobot}{a}
\newcommand{\Actionrobot}{A}

\newcommand{\userinputtoaction}{\mathcal{D}}

\newcommand{\state}{s}
\newcommand{\State}{S}
\newcommand{\belief}{b}

\newcommand{\staterobot}{x}
\newcommand{\Staterobot}{X}

\newcommand{\transition}{T}

\newcommand{\pomdpohm}{\Omega}

\newcommand{\costrobot}{C^\robottext}
\newcommand{\costusergoal}{C^\usertext_\goal}

\newcommand{\qpomdp}{Q}

\newcommand{\qmdp}{Q_{\goal}}
\newcommand{\vmdp}{V_{\goal}}
\newcommand{\qmdpt}[1]{Q_{\goal}^{#1}}
\newcommand{\vmdpt}[1]{V_{\goal}^{#1}}
\newcommand{\qsoft}{Q^{\approx}}
\newcommand{\vsoft}{V^{\approx}}
\newcommand{\qgoal}{\qmdp}
\newcommand{\vgoal}{\vmdp}
\newcommand{\qgoalsoft}{\qsoft_\goal}
\newcommand{\vgoalsoft}{\vsoft_\goal}

\newcommand{\qgoalsoftt}[1]{\qsoft_{ \goal,{#1}}}
\newcommand{\vgoalsoftt}[1]{\vsoft_{ \goal,{#1}}}
\newcommand{\costgoal}{C_\goal}

\newcommand{\costgoaluser}{\costgoal^\usertext}

\newcommand{\target}{\kappa}
\newcommand{\Target}{K}
\newcommand{\costtarg}{C_{\target}}
\newcommand{\costtargprime}{C_{\target'}}
\newcommand{\costtargstar}{C_{\target^*}}
\newcommand{\costtargrobot}{C_{\target}^{\robottext}}

\newcommand{\costtarguser}{C_{\target}^{\usertext}}

\newcommand{\qtarg}{Q_{\target}}
\newcommand{\vtarg}{V_{\target}}

\newcommand{\vtargt}[1]{V_{\target}^{#1}}
\newcommand{\vtargstar}{V_{\target^*}}
\newcommand{\vtargstart}[1]{V_{\target^*}^{#1}}
\newcommand{\qtargsoft}{Q^{\approx}_{\target}}
\newcommand{\vtargsoft}{V^{\approx}_{\target}}
\newcommand{\qtargsoftt}[1]{\qsoft_{ \target,{#1}}}
\newcommand{\vtargsoftt}[1]{\vsoft_{ \target,{#1}}}

\newcommand{\traj}{\xi}

\newcommand{\trajat}[1]{\traj_{#1}^{t \rightarrow T}}
\newcommand{\trajatp}[1]{\traj_{#1}^{t+1 \rightarrow T}}
\newcommand{\trajtot}{\traj^{0 \rightarrow t}}

\newcommand{\sref}[1]{Sec.~\ref{#1}}
\newcommand{\figref}[1]{Fig.~\ref{#1}}

\newtheorem{theorem}{Theorem}%[section]

\DeclareMathOperator*{\argmin}{arg\,min}
\DeclareMathOperator*{\argmax}{arg\,max}
\DeclareMathOperator*{\softmin}{\text{soft}\!\min}

\newcommand{\xxnote}[3]{}
\ifx\hidenotes\undefined
  \usepackage{color}
  \renewcommand{\xxnote}[3]{\color{#2}{#1: #3}}
\fi

\title{Shared Autonomy via Hindsight Optimization}
\author{ \parbox{\linewidth}{\centering Shervin Javdani, Siddhartha S. Srinivasa, J. Andrew Bagnell\\
         The Robotics Institute, 
         Carnegie Mellon University\\
         {\tt \small \{sjavdani, siddh, dbagnell\}@cs.cmu.edu}}
}

\date{January 2015}

\begin{document}

\maketitle

\begin{abstract}
In \emph{shared autonomy}, user input and robot autonomy are combined to control a robot to achieve a goal. Often, the robot does not know a priori which goal the user wants to achieve, and must both predict the user's intended goal, and assist in achieving that goal. We formulate the problem of shared autonomy as a Partially Observable Markov Decision Process with uncertainty over the user's goal. We utilize maximum entropy inverse optimal control to estimate a distribution over the user's goal based on the history of inputs. Ideally, the robot assists the user by solving for an action which minimizes the expected cost-to-go for the (unknown) goal. As solving the POMDP to select the optimal action is intractable, we use hindsight optimization to approximate the solution. In a user study, we compare our method to a standard predict-then-blend approach. We find that our method enables users to accomplish tasks more quickly while utilizing less input. However, when asked to rate each system, users were mixed in their assessment, citing a tradeoff between maintaining control authority and accomplishing tasks quickly.
\end{abstract}

\IEEEpeerreviewmaketitle

\section{Introduction}
\label{sec:intro}
Robotic teleoperation enables a user to achieve their intended goal by providing inputs into a robotic system. In direct teleoperation, user inputs are mapped directly to robot actions, putting the burden of control entirely on the user. However, input interfaces are often noisy, and may have fewer degrees of freedom than the robot they control. This makes operation tedious, and many goals impossible to achieve. \emph{Shared Autonomy} seeks to alleviate this by combining teleoperation with autonomous assistance.
%These systems are limited by inadequacies and noise in the input interface, making operation tedious, and many goals impossible to achieve. \emph{Shared autonomy} systems attempt to overcome these deficiencies by blending user input with autonomous assistance, decreasing the user's burden of control.
%It has applications in assistive robots, industrial assembly, remotely operated vehicles, and surgical robots. 
%Where direct teleoperation falls short to to noisy interfaces, tedious tasks, assymetry of control, 

A key challenge in shared autonomy is that the system may not know a priori which goal the user wants to achieve. 
%This is especially true in the case of continuous goal sets, for example the precise grasp point on a drill or a coffee mug.
Thus, many prior works~\cite{kofman_2005, aarno_2005_virtualfixtures, yu_2005, dragan_2013_assistive} split shared autonomy into two parts: 1) predict the user's goal, and 2) assist for that single goal, potentially using prediction confidence to regulate assistance. We refer to this approach as \emph{predict-then-blend}.

In contrast, we follow more recent work~\cite{hauser_2013} which assists for an entire distribution over goals, enabling assistance even when the confidence for any particular goal is low. This is particularly important in cluttered environments, where it is difficult - and sometimes impossible - to predict a single goal. 
%\ssnote{what's wrong with the state of the art? what are we fixing?}
%The problem of assistance in shared autonomy can be represented as a cooperative game,%stochastic game?,
%where the user and robot are both making inferences about one another when deciding how to act. Solving games such as this are computationally intractable.

%Implicit in our model is the assumption that the user is executing a policy for their known goal \emph{without} knowledge of the robot assistance (\figref{fig:robot_human_model}). Otherwise, we would have to recursively nest beliefs resulting in a stochastic game.

We formalize shared autonomy by modeling the system's task as a Partially Observable Markov Decision Process (POMDP)~\cite{smallwood_1973_pomdp, kaelbling_1998_pomdp} with uncertainty over the user's goal. We assume the user is executing a policy for their known goal \emph{without} knowledge of assistance. In contrast, the system models both the user input and robot action, and solves for an assistance action that minimizes the total expected cost-to-go of both systems. See \figref{fig:robot_human_model}. %TODO doesn't actually minimize this

\begin{figure}[t]
\centering
\includegraphics[width=0.49\textwidth, trim=28 53 18 2, clip=true]{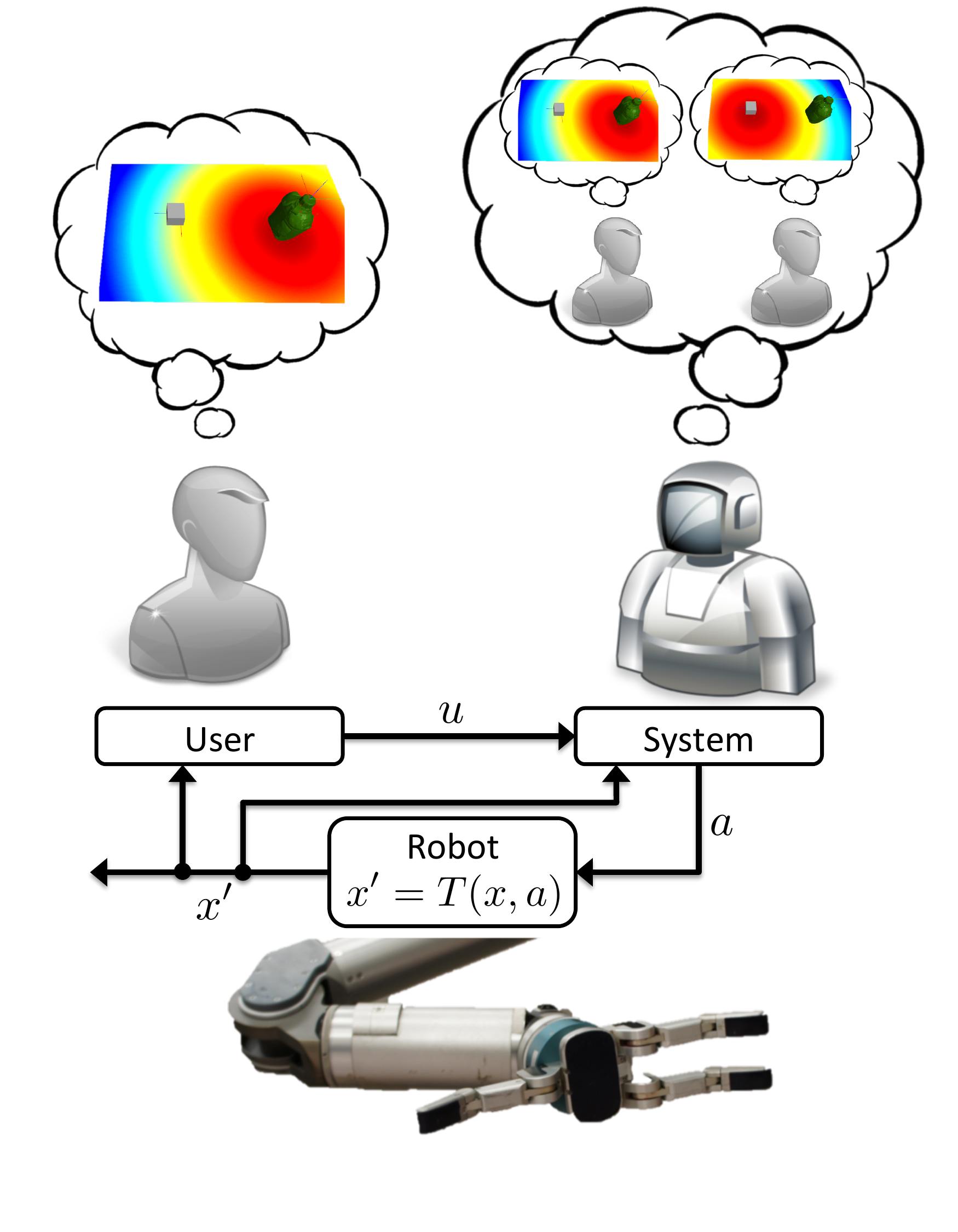}
\caption{ Our shared autonomy framework. We assume the user is executing a stochastically optimal policy for a known goal, without knowledge of assistance. We depict this single-goal policy as a heatmap plotting the value function at each position. Here, the user's target is the canteen. The shared autonomy system models all possible goals and their corresponding policies. From user inputs $\actionuser$, a distribution over goals is inferred. Using this distribution and the value functions for each goal, an action $\actionrobot$ is executed on the robot, transitioning the robot state from $\staterobot$ to $\staterobot'$. The user and shared autonomy system both observe this state, and repeat action selection.}
 %The underlying assumption for our shared autonomy model. Given two goals (e.g. an object to grasp), the user follows a policy to optimize a cost function for their known goal, without knowledge of robot assistance. The robot, in contrast, models the user's potential goals and their corresponding cost functions, both inferring the user's goal, and assisting to achieving it. As such, both agents work in tandem to control the robot.
 \label{fig:robot_human_model}
\end{figure}

%\begin{figure}[t]
%\centering
%%\includegraphics[width=0.49\textwidth, trim=225 110 530 325, clip=true]{images/robot_model.png}
%\includegraphics[width=0.49\textwidth, trim=0 0 0 0, clip=true]{images/robot_model_real_cropped.jpg}
% \caption{ The underlying assumption for our shared autonomy model. Given two goals (e.g. an object to grasp), the user follows a policy to optimize a cost function for their known goal, without knowledge of robot assistance. The robot, in contrast, models the user's potential goals and their corresponding cost functions, both inferring the user's goal, and assisting to achieving it. As such, both agents work in tandem to control the robot.}
% \label{fig:robot_human_model}
%\end{figure} 
%
%\begin{figure}[t]
%\centering
%%\includegraphics[width=0.49\textwidth, trim=225 110 530 325, clip=true]{images/robot_model.png}
%\includegraphics[width=0.49\textwidth, trim=0 0 0 0, clip=true]{images/robot_model_cropped.jpg}
% \caption{ Same as other caption.}
% \label{fig:robot_human_model_2}
%\end{figure} 

The result is a system that will assist for any distribution over goals. When the system is able to make progress for all goals, it does so automatically. When a good assistance strategy is ambiguous (e.g. the robot is in between two goals), the output can be interpreted as a blending between user input and robot autonomy based on confidence in a particular goal, which has been shown to be effective~\cite{dragan_2013_assistive}. See \figref{fig:teledata}.

Solving for the optimal action in our POMDP is intractable. Instead, we approximate using QMDP~\cite{littman_1995}, also referred to as hindsight optimization~\cite{chong_2000,yoon_2008}. This approximation has many properties suitable for shared autonomy: it is computationally efficient, works well when information is gathered easily~\cite{koval_2014}, and will not oppose the user to gather information.

Additionally, we assume each goal consists of multiple targets (e.g. an object has multiple grasp poses), of which any are acceptable to a user with that goal. Given a known cost function for each target, we derive an efficient computation scheme for goals that decomposes over targets.%for both assistance and prediction.

To evaluate our method, we conducted a user study where users teleoperated a robotic arm to grasp objects using our method and a standard predict-then-blend approach. Our results indicate that users accomplished tasks significantly more quickly with less control input with our system. However, when surveyed, users tended towards preferring the simpler predict-then-blend approach, 
%mentioning that they disliked the lack of control they felt while using our system. 
citing a trade-off between control authority and efficiency. 
We found this surprising, as prior work indicates that task completion time correlates strongly with user satisfaction, even at the cost of control authority~\cite{dragan_2013_assistive, hauser_2013, gombolay_2014}. We discuss potential ways to alter our model to take this into account.

%Works that commit to a goal:
%\cite{kofman_2005}
%Dont consider multiple goals, just one fixed autonomy system: 
%\cite{aigner_1997, crandall_2002, hauser_2011}
%Works that estimate one goal, and assist for that goal (with some confidence):
%\cite{aarno_2005_virtualfixtures, yu_2005, dragan_2013_assistive}
%Plan with entire distribution
%\cite{hauser_2013}

\section{Related Works}
\label{sec:related}
We separate related works into goal prediction and assistance strategies.

\subsection{Goal Prediction}
Maximum entropy inverse optimal control (MaxEnt IOC) methods have been shown to be effective for goal prediction~\cite{ziebart_2008, ziebart_2009, ziebart_2012, dragan_2013_assistive}. In this framework, the user is assumed to be an intent driven agent approximately optimizing a cost function. By minimizing the worst-case predictive loss, Ziebart et al.~\cite{ziebart_2008} derive a model where trajectory probability decreases exponentially with cost, and show how this cost function can be learned efficiently from user demonstrations. They then derive a method for inferring a distribution over goals from user inputs, where probabilities correspond to how efficiently the inputs achieve each goal~\cite{ziebart_2009}. While our framework allows for any prediction method, we choose to use MaxEnt IOC, as we can directly optimize for the user's cost in our POMDP.
%We utilize this method for prediction, as we can optimize for the user's cost function directly in our POMDP.

Others have approached the prediction problem by utilizing various machine learning methods. Koppula and Saxena~\cite{koppula_2013} extend conditional random fields (CRFs) with object affordances to predict potential human motions. Wang et al.~\cite{wang_2013_intentioninference} learn a generative predictor by extending Gaussian Process Dynamical Models (GPDMs) with a latent variable for intention. Hauser~\cite{hauser_2013} utilizes a Gaussian mixture model over task types (e.g. reach, grasp), and predicts both the task type and continuous parameters for that type (e.g. movements) using Gaussian mixture autoregression.

\subsection{Assistance Methods}
Many prior works assume the user's goal is known, and study how methods such as potential fields~\cite{aigner_1997, crandall_2002} and motion planning~\cite{hauser_2011} can be utilized to assist for that goal.

For multiple goals, many works follow a predict-then-blend approach of predicting the most likely goal, then assisting for that goal. These methods range from taking over when confident~\cite{fagg_2004, kofman_2005}, to virtual fixtures to help follow paths~\cite{aarno_2005_virtualfixtures}, to blending with a motion planner~\cite{dragan_2013_assistive}. Many of these methods can be thought of as an \emph{arbitration} between the user's policy and a fully autonomous policy for the most likely goal~\cite{dragan_2013_assistive}. These two policies are blended, where prediction confidence regulates the amount of assistance.

%The parameters of this blending are dependent on the user, task, and specifics of the system~\cite{dragan_2013_assistive}. When an assistance strategy is ambiguous, our policy looks similar to this blending, with parameters set automatically by the system. When an action can assist most goals, our system will automatically take that action, enabling a greater level of assistance where it makes sense.

Recently, Hauser~\cite{hauser_2013} presented a system which provides assistance while reasoning about the entire distribution over goals. Given the current distribution, the planner optimizes for a trajectory that minimizes the expected cost, assuming that no further information will be gained. After executing the plan for some time, the distribution is updated by the predictor, and a new plan is generated for the new distribution. In order to efficiently compute the trajectory, it is assumed that the cost function corresponds to squared distance, resulting in the calculation decomposing over goals. In contrast, our model is more general, enabling any cost function for which a value function can be computed. Furthermore, our POMDP model enables us to reason about future human actions.
%Potential note: squared distance likely a bad model for shared autonomy, since it will move very quickly when far. You would, at least, want to impose a velocity constraint.

Planning with human intention models has been used to avoid moving pedestrians. Ziebart et al.~\cite{ziebart_2009} use MaxEnt IOC to learn a predictor of pedestrian motion, and use this to predict the probability a location will be occupied at each time step. They build a time-varying cost map, penalizing locations likely to be occupied, and optimize trajectories for this cost. Bandy et al.~\cite{bandy_2012} use fixed models for pedestrian motions, and focus on utilizing a POMDP framework with SARSOP~\cite{kurniawati_2008} for selecting good actions. Like our approach, this enables them to reason over the entire distribution of potential goals. They show this outperforms utilizing only the maximum likelihood estimate of goal prediction for avoidance. 

Outside of robotics, Fern and Tadepalli~\cite{fern_2010} have studied MDP and POMDP models for assistance. Their study focuses on an interactive assistant which suggest actions to users, who then accept or reject the action. They show that optimal action selection even in this simplified model is PSPACE-complete. However, a simple greedy policy has bounded regret.

Nguyen et al.~\cite{nguyen_2011} and Macindoe et al.~\cite{macindoe_2012} apply similar models to creating agents in cooperative games, where autonomous agents simultaneously infer human intentions and take assistance actions. Here, the human player and autonomous agent each control separate characters, and thus affect different parts of state space. Like our approach, they model users as stochastically optimizing an MDP, and solve for assistance actions with a POMDP. In contrast to these works, our action space and state space are continuous.

\section{Problem Statement} 
\label{sec:problem}

\begin{figure}[t]
\centering
 \begin{subfigure}{0.240\textwidth}
   \centering 
   \hspace*{1mm} \includegraphics[width=0.92\textwidth, trim=400 50 400 450, clip=true]{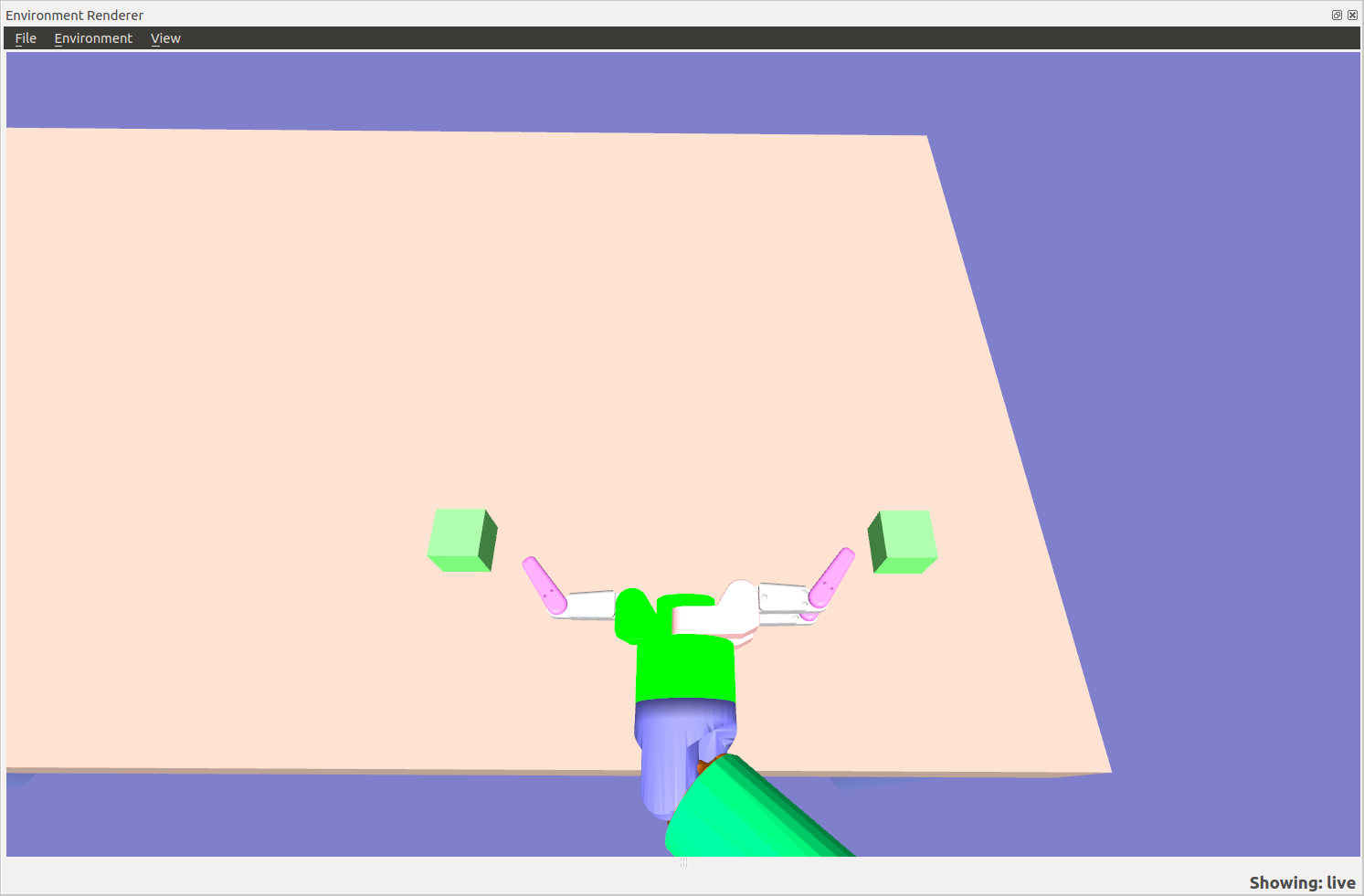}
   \hspace*{-2mm} \includegraphics[trim=1 3 0 -3, clip=true]{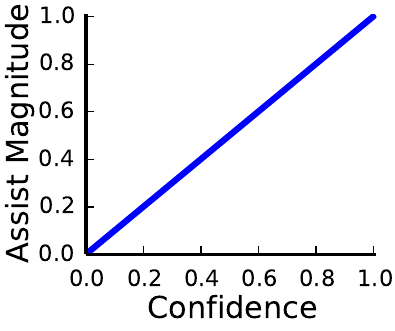} \hfill
   \caption{}
 \label{fig:teledata_cen}
 \end{subfigure}
 \hfill
 \begin{subfigure}{0.240\textwidth}
   \centering
   \hspace*{1mm} \includegraphics[width=0.92\textwidth, trim=400 50 400 450, clip=true]{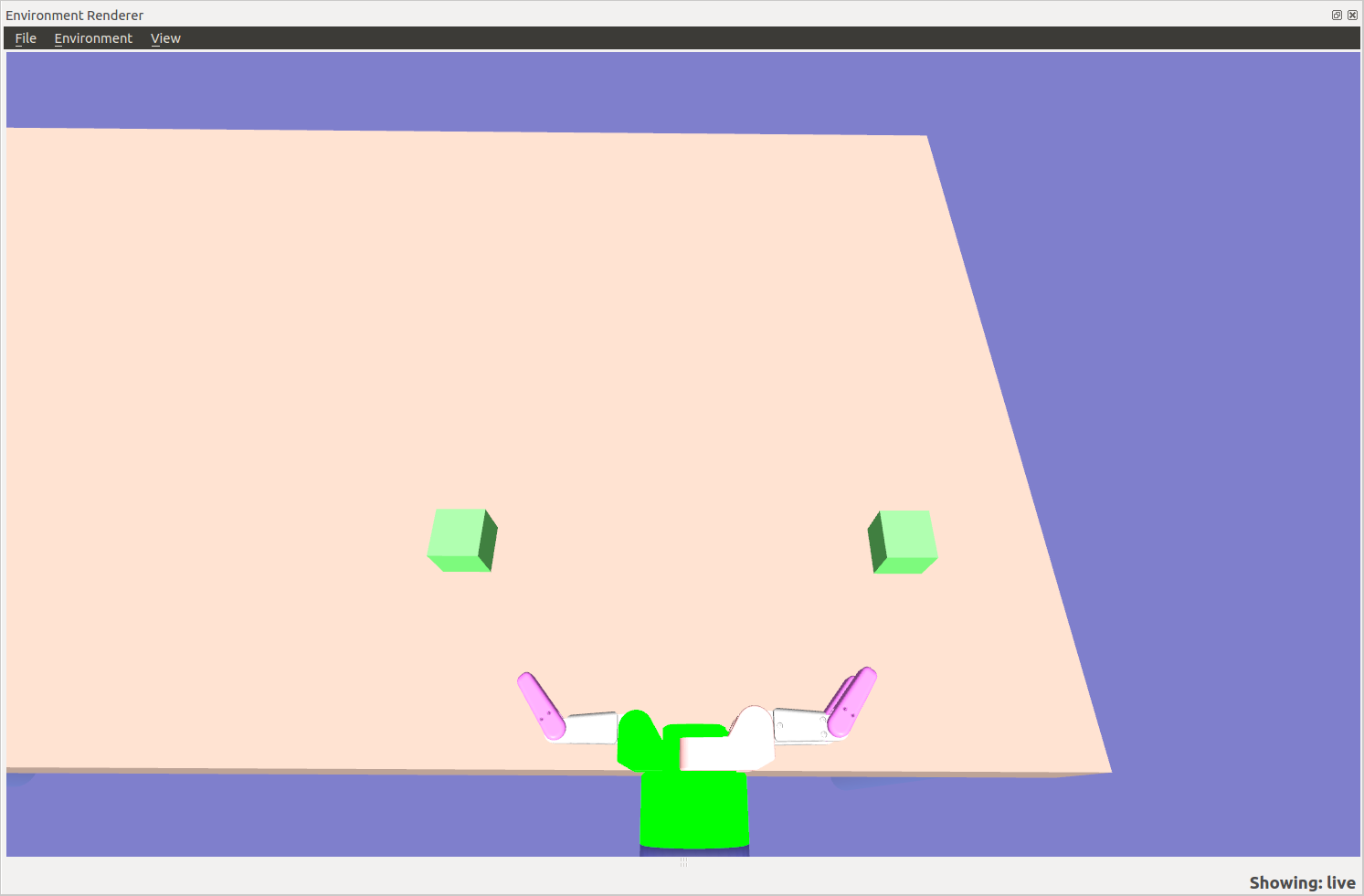}
   \hspace*{-1mm}\includegraphics[trim=1 3 0 -3, clip=true]{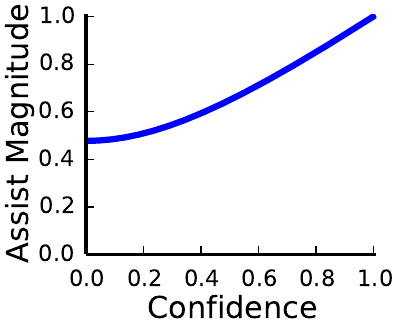} \hfill
  \caption{}
 \label{fig:teledata_back}
 \end{subfigure}
 \caption{\footnotesize Arbitration as a function of confidence with two goals. Confidence $=\max_g p(g) - \min_g p(g)$, which ranges from $0$ (equal probability) to $1$ (all probability on one goal). (\subref{fig:teledata_cen}) The hand is directly between the two goals, where no action assists for both goals. As confidence for one goal increases, assistance increases linearly. (\subref{fig:teledata_back}) From here, going forward assists for both goals, enabling the assistance policy to make progress even with $0$ confidence.} \label{fig:teledata}
\end{figure} 

%\begin{table}[t!]
%\centering
%\begin{tabular}{rl}
%\toprule
%Symbol & Description\\ \midrule
%$\goal \in \Goal$ & Discrete goal\\
%$\staterobot \in \Staterobot$ & Continuous robot state\\
%$\state \in \State$ & Continuous robot state and discrete goals\\
%$\actionrobot \in \Actionrobot$ & Continuous robot action\\
%$\actionuser \in \Actionuser$ & Continuous user input\\
%$\userinputtoaction(\actionuser) \in \Actionrobot$ & Direct teleoperation action from $\actionuser$\\
%$\transition(\staterobot, \actionuser, \staterobot')$ & Transition function \\
%$\state \in \Staterobot$ & System state, where $\state = (\staterobot, \goal)$ and $ \State = \Staterobot \times \Goal$\\
%$\costusergoal(\staterobot, \actionuser)$ & User cost of input $\actionuser$ in robot state $\staterobot$ with goal $\goal$\\
%$\costrobot(\state, \actionrobot, \actionuser)$ & System cost of action $\actionrobot$ in state $\state$ with user input $\actionuser$\\
%\bottomrule
%\end{tabular}
%\caption{Symbols and definitions}
%\end{table}

%some high level thing about robot system, user mdp, system pomdp?
We assume there are a discrete set of possible goals, one of which is the user's intended goal. The user supplies inputs through some interface to achieve their goal. Our shared autonomy system does not know the intended goal a priori, but utilizes user inputs to infer the goal. It selects actions to minimize the expected cost of achieving the user's goal.
%The user supplies inputs through some interface, which they believe will directly be translated to changes in the robot state. They do so to move the robot to some known goal. 

Formally, let $\staterobot \in \Staterobot$ be the continuous robot state (e.g. position, velocity), and let $\actionrobot \in \Actionrobot$ be the continuous actions (e.g. velocity, torque). We model the robot as a deterministic dynamical system with transition function $\transition: \Staterobot \times \Actionrobot \rightarrow \Staterobot$. %  where applying action $\actionrobot$ in state $\staterobot$ results in state $\staterobot'$.
The user supplies continuous inputs $\actionuser \in \Actionuser$ via an interface (e.g. joystick, mouse). These user inputs map to robot actions through a known deterministic function $\userinputtoaction: \Actionuser \rightarrow \Actionrobot$, corresponding to the effect of \emph{direct teleoperation}.

In our scenario, the user wants to move the robot to one goal in a discrete set of goals $\goal \in \Goal$. We assume access to a stochastic user policy for each goal $\policyusergoal(\staterobot) = p(\actionuser | \staterobot, \goal)$, usually learned from user demonstrations. %Here, the user assumes inputs get mapped directly to actions through $\userinputtoaction$ - thus, they assume direct teleoperation.
In our system, we model this policy using the maximum entropy inverse optimal control (MaxEnt IOC)~\cite{ziebart_2008} framework, which assumes the user is approximately optimizing some cost function for their intended goal $g$, $\costusergoal: \Staterobot \times \Actionuser \rightarrow \mathcal{R}$. This model corresponds to a goal specific Markov Decision Process (MDP), defined by the tuple $\left(\Staterobot, \Actionuser, \transition, \costusergoal\right)$. We discuss details in \sref{sec:prediction}.

%We assume the user is executing some stochastic policy $\policyusergoal(\staterobot) = p(\actionuser | \staterobot, \goal)$, corresponding to approximately optimizing a known cost function for their intended goal $\costusergoal$. We assume this policy does not have knowledge of the assistance, and utilizes $\userinputtoaction$ to map user inputs to changes in robot state. We discuss how $\costusergoal$ can be learned from demonstration in \sref{sec:prediction}. Formally, this corresponds to modelling the user as a Markov Decision Process (MDP) for a particular goal, defined by the tuple $\left(\Staterobot, \Actionuser, \transition, \costusergoal\right)$. %include \userinputtoaction in tuple?

Unlike the user, our system does not know the intended goal. We model this with a Partially Observable Markov Decision Process (POMDP) with uncertainty over the user's goal. A POMDP maps a distribution over states, known as the \emph{belief} $\belief$, to actions. Define the system state $\state \in \State$ as the robot state augmented by a goal, $\state = (\staterobot, \goal)$ and $\State = \Staterobot \times \Goal$. In a slight abuse of notation, we overload our transition function such that $\transition: \State \times \Actionrobot \rightarrow \State$, which corresponds to transitioning the robot state as above, but keeping the goal the same.
%with $\state = (\staterobot, \goal)$ and $\state' = (\staterobot', \goal)$. This corresponds to transitioning the robot state as $\transition(\staterobot, \actionrobot, \staterobot')$, and keeping the goal $\goal$ the same.

In our POMDP, we assume the robot state is known, and all uncertainty is over the user's goal. Observations in our POMDP correspond to user inputs $\actionuser \in \Actionuser$. Given a sequence of user inputs, we infer a distribution over system states (equivalently a distribution over goals) using an observation model $\pomdpohm$. This corresponds to computing $\policyusergoal(\staterobot)$ for each goal, and applying Bayes' rule. We provide details in \sref{sec:prediction}.

The system uses cost function $\costrobot: \State \times \Actionrobot \times \Actionuser \rightarrow \mathcal{R}$, corresponding to the cost of taking robot action $\actionrobot$ when in system state $\state$ and the user has input $\actionuser$. Note that allowing the cost to depend on the observation $\actionuser$ is non-standard, but important for shared autonomy, as prior works suggest that users prefer maintaining control authority~\cite{kim_2012}. This formulation enables us to penalize robot actions which deviate from $\userinputtoaction(\actionuser)$. Our shared autonomy POMDP is defined by the tuple $\left(\State, \Actionrobot, \transition, \costrobot, \Actionuser, \pomdpohm \right)$. The optimal solution to this POMDP minimizes the expected accumulated cost $\costrobot$. As this is intractable to compute, we utilize Hindsight Optimization to select actions, described in \sref{sec:hindsight}.

\begin{figure*}[t]
\centering
% \begin{subfigure}{0.24\textwidth}
%   \centering 
%   \includegraphics[width=1.0\textwidth, trim=250 150 200 190, clip=true]{images/valfunc_left_noback/rss_1.png}
%  \includegraphics[width=0.7\textwidth, trim=150 515 640 150, clip=true]{images/valuefunc_text/valuefunc_slides_1.pdf}
%   \caption{ }
% \label{fig:valfunc_1}
% \end{subfigure}
 \begin{subfigure}{0.32\textwidth}
   \centering 
   \begin{tikzpicture}[every node/.style={anchor=south west,inner sep=0pt}, x=1mm, y=1mm,]    
     \node {\includegraphics[width=1.0\textwidth, trim=250 150 200 190, clip=true]{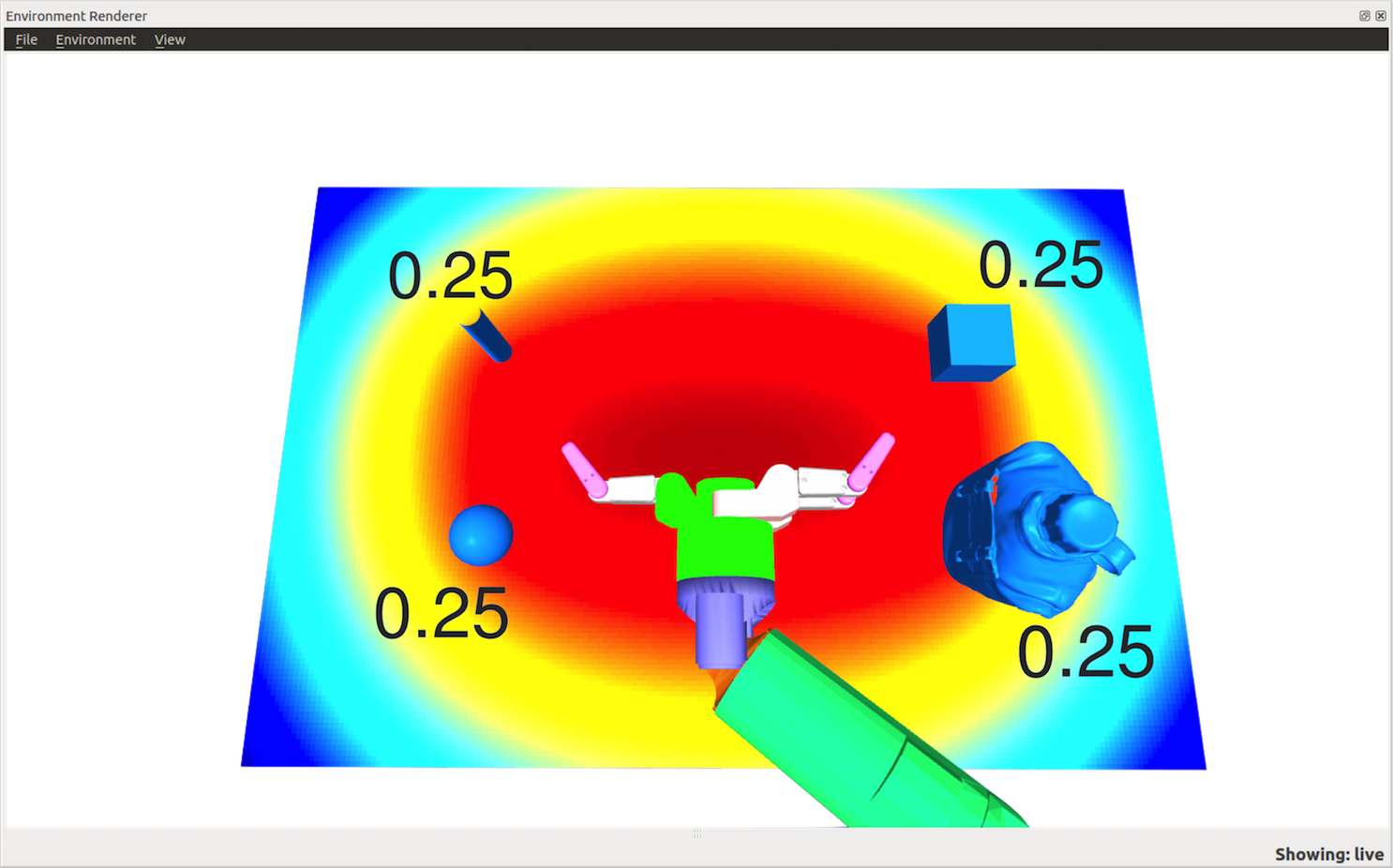}} ;
     \node [opacity=0.6]{\includegraphics[width=1.0\textwidth, trim=250 150 200 190, clip=true]{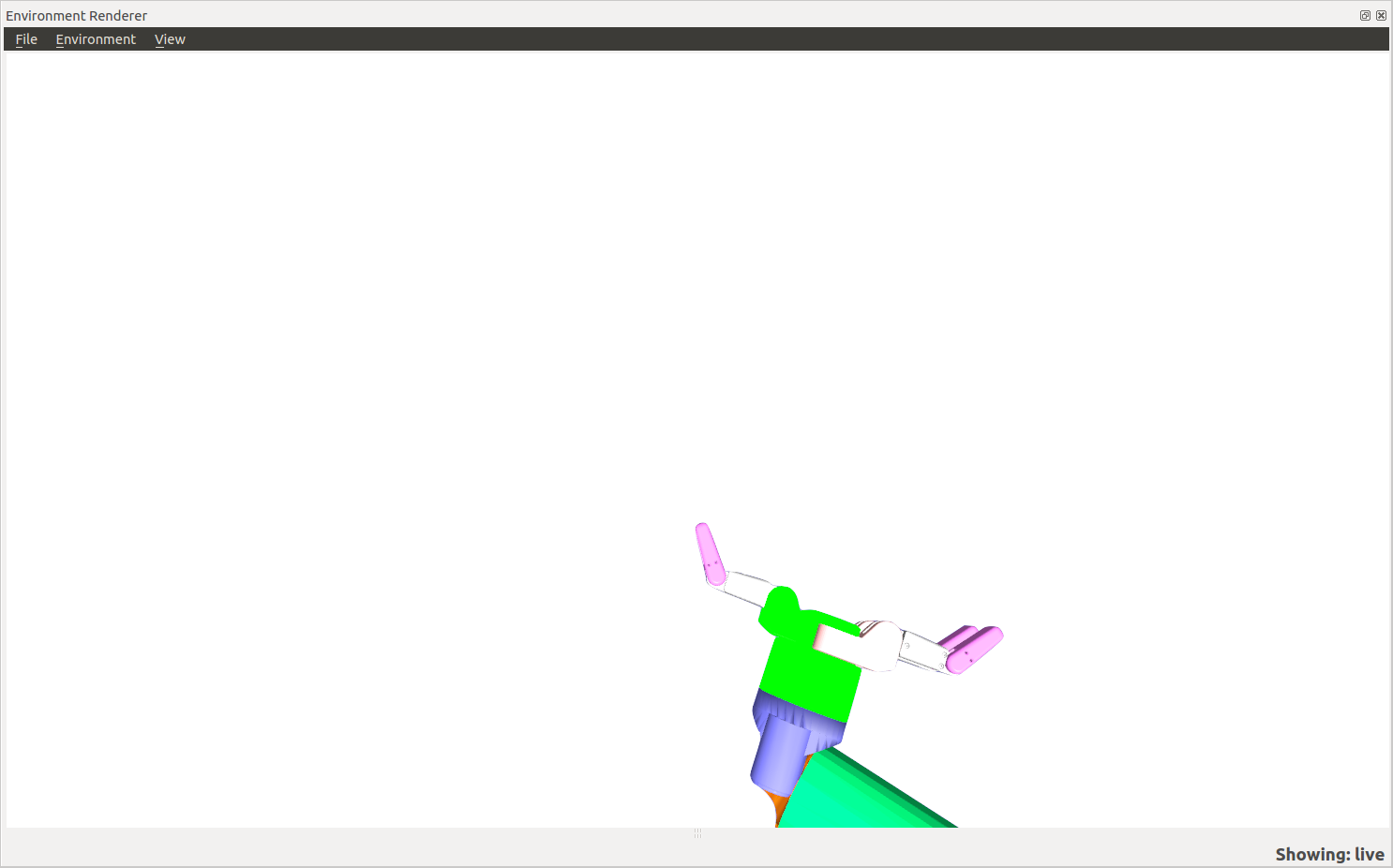} };
 \end{tikzpicture}
  \includegraphics[width=0.7\textwidth, trim=150 515 640 150, clip=true]{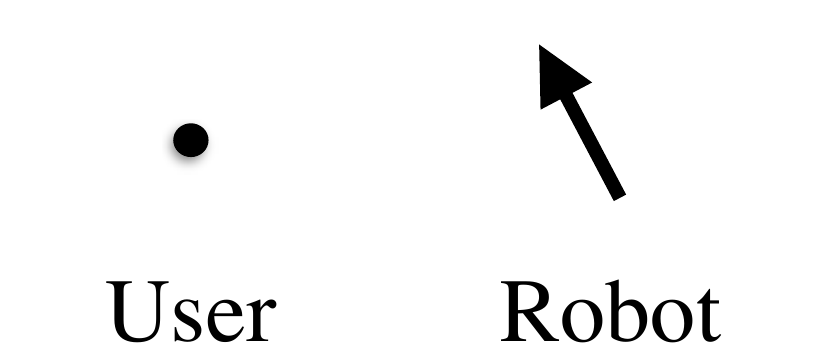}
   \caption{ }
 \label{fig:valfunc_2}
 \end{subfigure}
 \begin{subfigure}{0.32\textwidth}
   \centering 
   \begin{tikzpicture}[every node/.style={anchor=south west,inner sep=0pt}, x=1mm, y=1mm,]    
     \node {\includegraphics[width=1.0\textwidth, trim=250 150 200 190, clip=true]{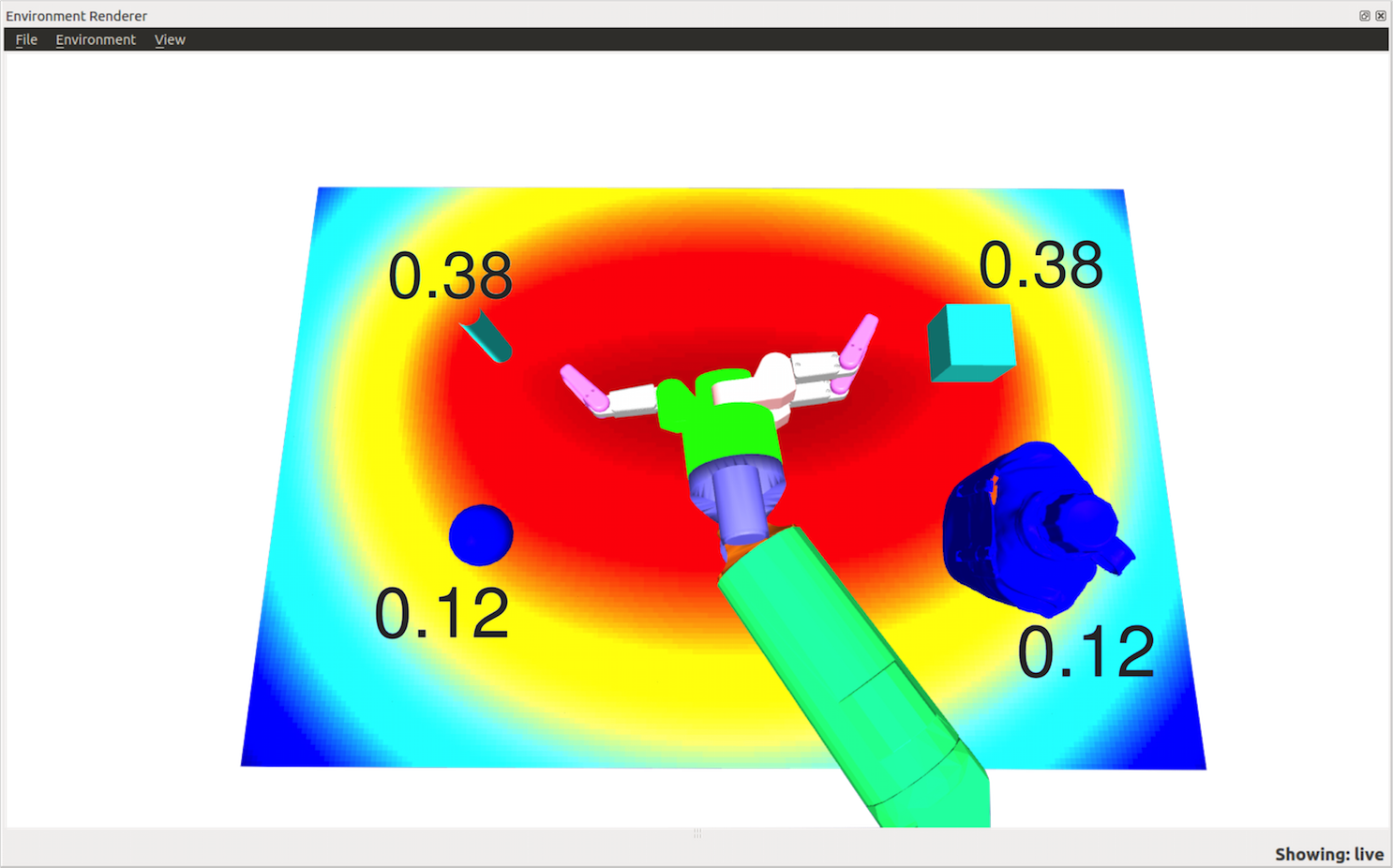}} ;
     \node [opacity=0.6]{\includegraphics[width=1.0\textwidth, trim=250 150 200 190, clip=true]{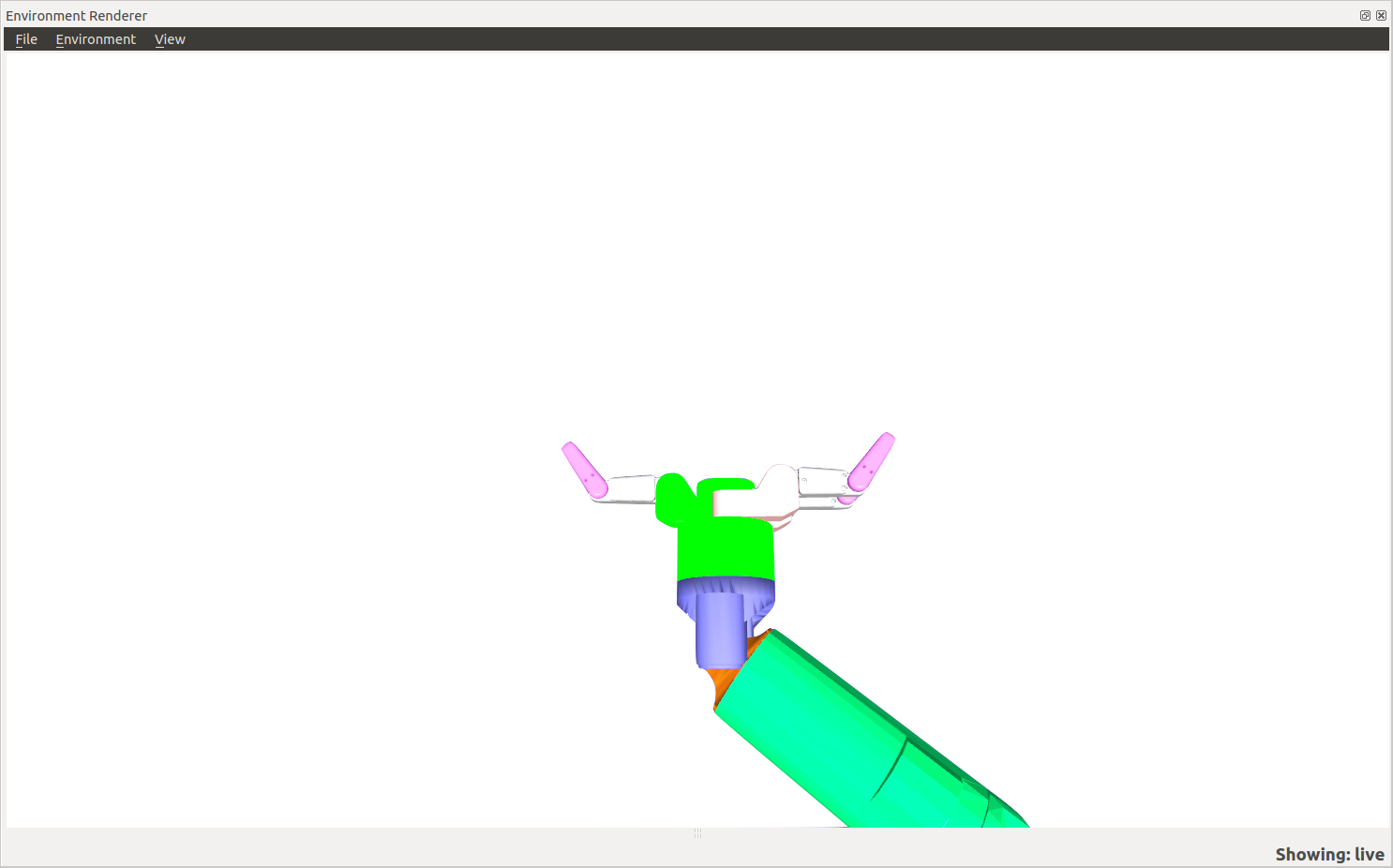} };
 \end{tikzpicture}
  \includegraphics[width=0.7\textwidth, trim=150 515 640 150, clip=true]{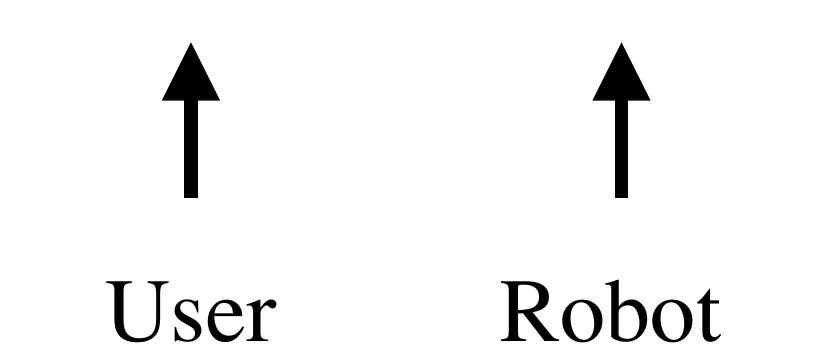}
 \caption{ }
 \label{fig:valfunc_3}
 \end{subfigure}
 \begin{subfigure}{0.32\textwidth}
   \centering 
   \begin{tikzpicture}[every node/.style={anchor=south west,inner sep=0pt}, x=1mm, y=1mm,]    
     \node {\includegraphics[width=1.0\textwidth, trim=250 150 200 190, clip=true]{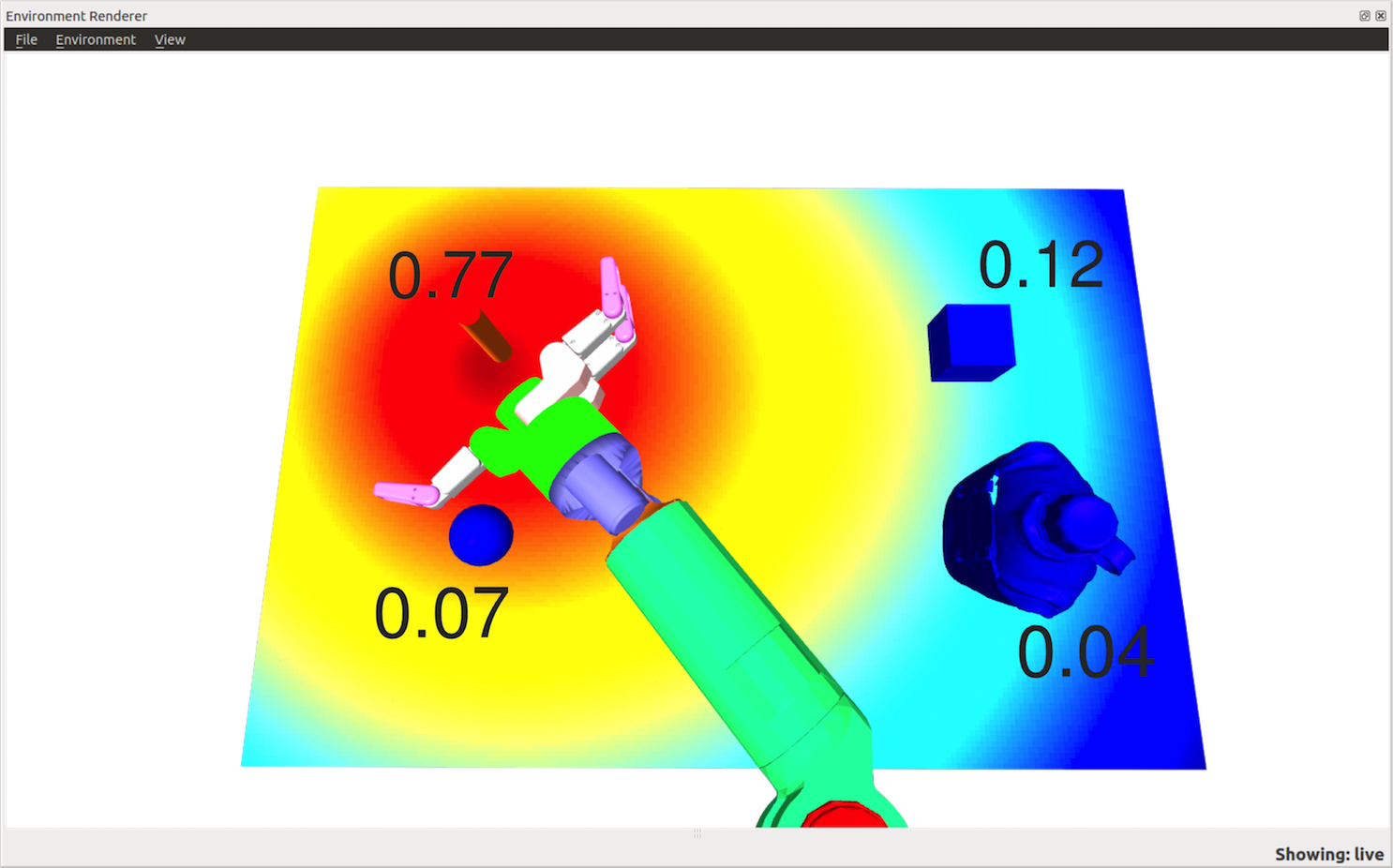}} ;
     \node [opacity=0.6]{\includegraphics[width=1.0\textwidth, trim=250 150 200 190, clip=true]{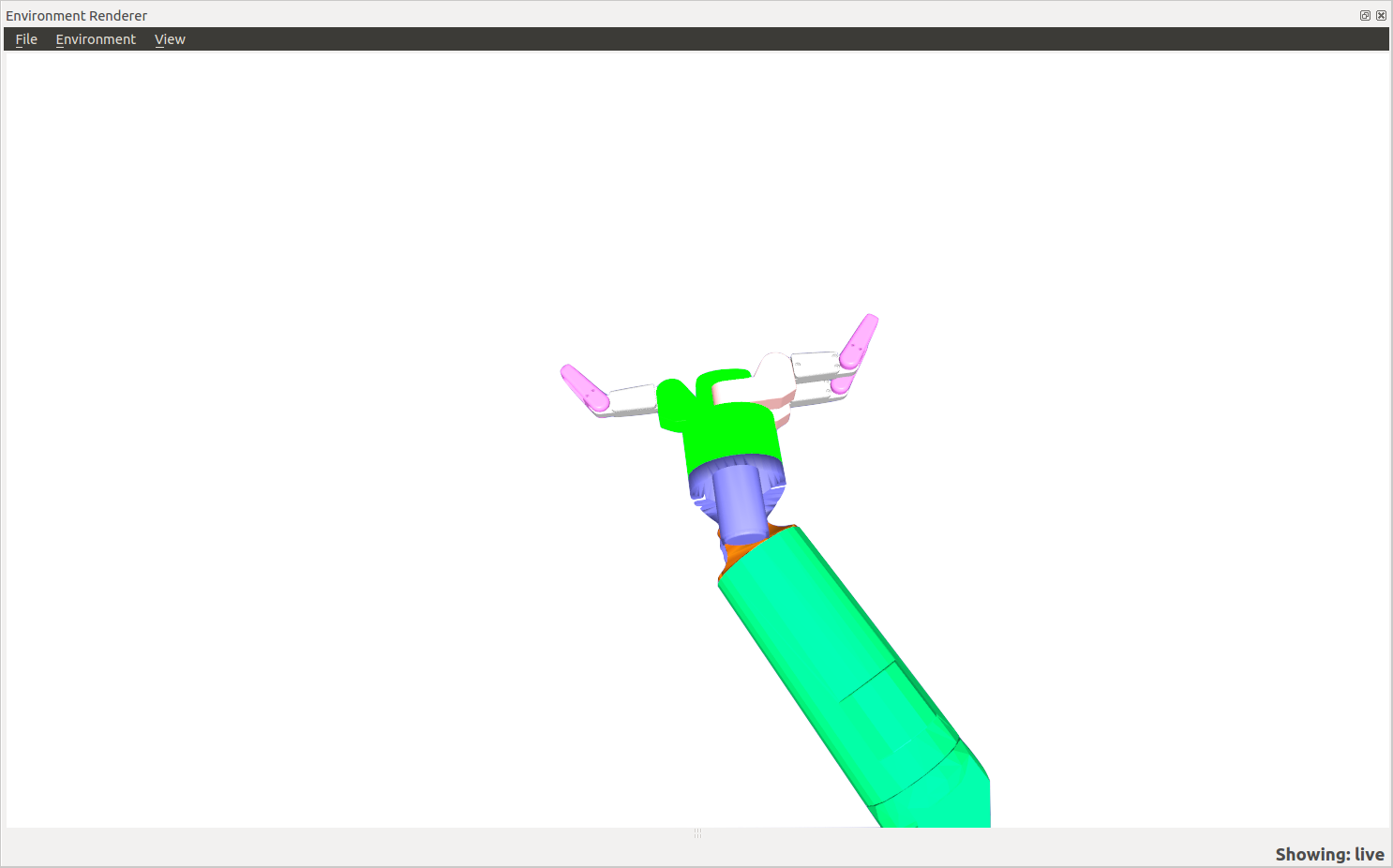} };
 \end{tikzpicture}
  \includegraphics[width=0.7\textwidth, trim=150 515 640 150, clip=true]{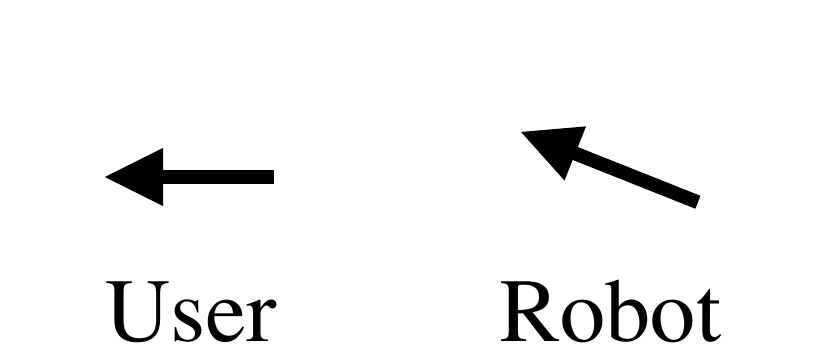}
 \caption{ }
 \label{fig:valfunc_4}
 \end{subfigure}
  \label{fig:valfunc}
  \caption{Estimated goal probabilities and value function for an object grasping trial. Top row: the probability of each goal object and a 2-dimensional slice of the estimated value function. The transparent end-effector corresponds to the initial state, and the opaque end-effector to the next state. Bottom row: the user input and robot control vectors which caused this motion. (\subref{fig:valfunc_2}) Without user input, the robot automatically goes to the position with lowest value, while estimated probabilities and value function are unchanged. (\subref{fig:valfunc_3}) As the user inputs ``forward'', the end-effector moves forward, the probability of goals in that direction increase, and the estimated value function shifts in that direction. (\subref{fig:valfunc_4}) As the user inputs ``left'', the goal probabilities and value function shift in that direction. Note that as the probability of one object dominates the others, the system automatically rotates the end-effector for grasping that object.}
\end{figure*}

\section{Modelling the user policy} 
\label{sec:prediction}

We now discuss our model of $\policyusergoal$. In principle, we could use any generative predictor~\cite{koppula_2013, wang_2013_intentioninference}. We choose to use maximum entropy inverse optimal control (MaxEnt IOC)~\cite{ziebart_2008}, as it explicitly models a user cost function $\costusergoal$. We can then optimize this directly by defining $\costrobot$ as a function of $\costusergoal$.

Define a sequence of robot states and user inputs as $\traj = \left\{ \staterobot_0, \actionuser_0, \cdots, \staterobot_T, \actionuser_T \right\}$. Note that sequences are not required to be trajectories, in that $\staterobot_{t+1}$ is not necessarily the result of applying $\actionuser_t$ in state $\staterobot_t$. Define the cost of a sequence as the sum of costs of all state-input pairs, $\costgoaluser(\traj) = \sum_{t} \costgoaluser(\staterobot_t, \actionuser_t)$. Let $\trajtot$ be a sequence from time $0$ to $t$, and $\trajat{\staterobot}$ a sequence of from time $t$ to $T$, starting at robot state $\staterobot$.

%Define the cost for a goal at a robot configuration as $\costgoaluser(\staterobot, \actionuser) = \costuser( \left\{\staterobot, \goal\right\}, \actionuser)$, and the cost of a sequence as the sum of costs for all configurations and inputs, $\costgoaluser(\traj) = \sum_{t} \costgoaluser(\staterobot_t, \actionuser_t)$. Let $\trajtot$ be a sequence from time $0$ to $t$, and $\trajat{\staterobot}$ a sequence of from time $t$ to $T$, starting at robot configuration $\staterobot$. %Using the principle of maximum entropy~\cite{ziebart_2008}, we compute the probability of a trajectory for a specific goal as $p(\traj | \goal) \propto \exp(-\costusergoal(\traj))$. That is, the probability of a trajectory decreases exponentially with cost. Importantly, one can also learn a cost function from demonstration to be consistent with this model~\cite{ziebart_2008}.
 %Using the principle of maximum entropy~\cite{ziebart_2008}, we compute the probability of a trajectory for a specific goal as $p(\traj | \goal) \propto \exp(-\costusergoal(\traj))$. That is, the probability of a trajectory decreases exponentially with cost. Importantly, one can also learn a cost function from demonstration to be consistent with this model~\cite{ziebart_2008}.

 It has been shown that minimizing the worst-case predictive loss results in a model where the probability of a sequence decreases exponentially with cost, $p(\traj | \goal) \propto \exp(-\costgoaluser(\traj))$~\cite{ziebart_2008}. Importantly, one can efficiently learn a cost function consistent with this model from demonstrations of user execution~\cite{ziebart_2008}.

Computationally, the difficulty lies in computing the normalizing factor $\int_{\traj} \exp(-\costgoaluser(\traj))$, known as the partition function. Evaluating this explicitly would require enumerating all sequences and calculating their cost.

However, as the cost of a sequence is the sum of costs of all state-action pairs, dynamic programming can be utilized to compute this through soft-minimum value iteration~\cite{ziebart_2009,ziebart_2012}:
\begin{align*}
  \qgoalsoftt{t}(\staterobot, \actionuser) = \costgoaluser(\staterobot, \actionuser) + \vgoalsoftt{t+1}(\staterobot')\\
  \vgoalsoftt{t}(\staterobot) = \softmin_{\actionuser} \qgoalsoftt{t}(\staterobot, \actionuser)
\end{align*}
Where $\staterobot '= \transition(\staterobot, \userinputtoaction(\actionuser))$, the result of applying $\actionuser$ at state $\staterobot$, and $\softmin_{x} f(x) = - \log \int_{x} \exp(-f(x)) dx$.

The log partition function is given by the soft value function, $\vgoalsoftt{t}(\staterobot) = - \log \int_{\trajat{\staterobot}} \exp\left(-\costgoaluser(\trajat{\staterobot})\right)$, where the integral is over all sequences starting at configuration $\staterobot$ and time $t$. Furthermore, the probability of a single input at a given configuration is given by $\policyuser_t(\actionuser | \staterobot, \goal) = \exp(\vgoalsoftt{t}(\staterobot) -\qgoalsoftt{t}(\staterobot, \actionuser))$~\cite{ziebart_2009}.

%make more clear that while our user policy doesn't consider robot assistance, it still affects this positive feedback thing
Many works derive a simplification that enables them to only look at the start and current configurations, ignoring the inputs in between~\cite{ziebart_2012, dragan_2013_assistive}. Key to this assumption is that $\traj$ corresponds to a trajectory, where applying action $\actionuser_t$ at $\staterobot_t$ results in $\staterobot_{t+1}$. However, if the system is providing assistance, this may not be the case. In particular, if the assistance strategy believes the user's goal is $\goal$, the assistance strategy will select actions to minimize $\costusergoal$. Applying these simplifications will result positive feedback, where the robot makes itself more confident about goals it already believes are likely. In order to avoid this, we ensure that the prediction probability comes from user inputs only, and not robot actions:
\begin{align*}
  p(\traj | \goal) &= \prod_t \policyuser_t(\actionuser_{t} | \staterobot_t, \goal)
\end{align*}
%Where the user applied input $\actionuser_t$ at state $\state_t$.
Finally, to compute the probability of a goal given the partial sequence up to $t$, we use Bayes' rule:
\begin{align*}
  p(\goal | \trajtot) &= \frac{p(\trajtot | \goal) p(\goal) }{\sum_{\goal'} p(\trajtot | \goal') p(\goal')}
\end{align*}
This corresponds to our POMDP observation model $\pomdpohm$.
%the prior sort of got tied into \traj in this formulation :/

%Following~\cite{dragan_2013_assistive}, we instead utilize a second order approximation about the optimal trajectory. Assuming a constant Hessian in the quadratic term allows the partition function to be evaluated analytically, giving us the probability over goals: 
%\begin{align*}
%  p(\xi_{S \rightarrow E} | G) &= \frac{ \exp\left(-c_g(\xi_{S \rightarrow E}) - c_g(\xi^*_{E \rightarrow G}) \right)} { \exp\left(-c_g(\xi^*_{S \rightarrow G})\right) },
%\end{align*}
%where $\xi^*_{X \rightarrow Y}$ is the optimal (minimum-cost) trajectory from $X$ to $Y$, $S$ is the starting pose of the end-effector, $E$ is the current end-effector pose and $G$ is the pose of the goal (object).
%Finally, Bayes' rule can be used to compute desired probability per goal. We assume that the only unobserved part of $\state$ corresponds to the goal $\stategoal$.
%\begin{align*}
%  p(\|\xi_{S\rightarrow E}) \propto p(\xi | \stategoal) p(G)
%\end{align*}

\section{Hindsight Optimization} 
\label{sec:hindsight}

Solving POMDPs, i.e. finding the optimal action for any belief state, is generally intractable. We utilize the QMDP approximation~\cite{littman_1995}, also referred to as hindsight optimization~\cite{chong_2000,yoon_2008} to select actions. The idea is to estimate the cost-to-go of the belief by assuming full observability will be obtained at the next time step. %compute the optimal action assuming that at the next state, full observability is obtained.
The result is a system that never tries to gather information, but can plan efficiently in the deterministic subproblems. This concept has been shown to be effective in other domains~\cite{yoon_2008, yoon_2007}.

We believe this method is suitable for shared autonomy for many reasons. Conceptually, we assume the user will provide inputs at all times, and therefore we gain information without explicit information gathering. In this setting, works in other domains have shown that QMDP performs similarly to methods that consider explicit information gathering~\cite{koval_2014}. Computationally, QMDP is efficient to compute even with continuous state and action spaces, enabling fast reaction to user inputs. Finally, explicit information gathering where the user is treated as an oracle would likely be frustrating~\cite{guillory_2011_noise, amershi_2014}, and this method naturally avoids it.

%say that this is a lower bound on cost-to-go?

Let $\qpomdp(\belief, \actionrobot, \actionuser)$ be the action-value function of the POMDP, estimating the cost-to-go of taking action $\actionrobot$ when in belief $\belief$ with user input $\actionuser$, and acting optimally thereafter. In our setting, uncertainty is only over goals, $\belief(\state) = \belief(\goal) = p(\goal | \trajtot)$.

Let $\qgoal(\staterobot, \actionrobot, \actionuser)$ correspond to the action-value for goal $\goal$, estimating the cost-to-go of taking action $\actionrobot$ when in state $\staterobot$ with user input $\actionuser$, and acting optimally for goal $\goal$ thereafter. The QMDP approximation is~\cite{littman_1995}:
\begin{align*}
  \qpomdp(\belief, \actionrobot, \actionuser) &= \sum_{\goal} \belief(\goal) \qgoal(\staterobot, \actionrobot, \actionuser)
\end{align*}

Finally, as we often cannot calculate $\argmax_{\actionrobot} \qpomdp(\belief, \actionrobot, \actionuser)$ directly, we use a first-order approximation, which leads to us to following the gradient of $\qpomdp(\belief, \actionrobot, \actionuser)$.

We now discuss two methods for approximating $\qgoal$:
\subsubsection{Robot and user both act}
Estimate $\actionuser$ with $\policyusergoal$ at each time step, and utilize $\costrobot(\{\staterobot, \goal\}, \actionrobot, \actionuser)$ for the cost. Using this cost, we could run q-learning algorithms to compute $\qgoal$. This would be the standard QMDP approach for our POMDP.

%FOR FUTURE SHERVIN we could seperate this into two, one where user is stochastic, the other where user is optimal

\subsubsection{Robot takes over}
Assume the user will stop supplying inputs, and the robot will complete the task. This enables us to use the cost function $\costrobot(\state, \actionrobot, \actionuser) = \costrobot(\state, \actionrobot, 0)$. Unlike the user, we can assume the robot will act optimally. Thus, for many cost functions we can analytically compute the value, e.g. cost of always moving towards the goal at some velocity.

An additional benefit of this method is that it makes no assumptions about the user policy $\policyusergoal$, making it more robust to modelling errors. We use this method in our experiments.

%In cases were an action exists to assist for all goals, this approximation will take that action. When there aren't any such actions, the output will look similar to a blending between the user control and our assistance strategy, solving for the parameters of blending based on the cost functions. This sort of blending has been shown to be effective in the past~\cite{dragan_2013_assistive}. See \figref{fig:teledata}.

%add something about 1st order approximation for continuous systems?

%Maybe more specifics for our system? 
%-First order approx for qmdp
%-we optimize directly for user's value function
%---actually, we aren't fully solving the POMDP assuming user is optimal

\section{Multi-Goal MDP} 
\label{sec:multigoal}

There are often multiple ways to achieve a goal. We refer to each of these ways as a \emph{target}. For a single goal (e.g. object to grasp), let the set of targets (e.g. grasp poses) be $\target \in \Target$. We assume each target has robot and user cost functions $\costtargrobot$ and $\costtarguser$, from which we compute the corresponding value and action-value functions $\vtarg$ and $\qtarg$, and soft-value functions $\vtargsoft$ and $\qtargsoft$. We derive the quantities for goals, $\vgoal, \qgoal, \vgoalsoft, \qgoalsoft$, as functions of these target functions.

\subsection{Multi-Target Assistance}
For simplicity of notation, let $\costgoal(\staterobot, \actionrobot) = \costrobot( \{\staterobot, \goal\}, \actionrobot, 0)$, and $\costtarg(\staterobot, \actionrobot) = \costtargrobot( \staterobot, \actionrobot, 0)$. We assign the cost of a state-action pair to be the cost for the target with the minimum cost-to-go after this state:
\begin{align*}
  \costgoal(\staterobot, \actionrobot) &= \costtargstar(\staterobot, \actionrobot) \qquad \target* = \argmin_\target \vtarg(\staterobot')
\end{align*}
Where $\staterobot'$ is the robot state when action $\actionrobot$ is applied at $\staterobot$.
\begin{theorem} \label{thm:mingoal_assist}
  Let $\vtarg$ be the value function for target $\target$. Define the cost for the goal as above. For an MDP with deterministic transitions, the value and action-value functions $\vmdp$ and $\qmdp$ can be computed as:
\begin{align*}
  \qgoal(\staterobot, \actionrobot) &= \costtargstar(\staterobot, \actionrobot) + \vtargstar(\staterobot') \qquad \target^* = \argmin \vtarg(\staterobot') \\
  \vgoal(\staterobot) &= \min_\target \vtarg(\staterobot)
\end{align*}
\end{theorem}
\begin{proof}
We show how the standard value iteration algorithm, computing $\qgoal$ and $\vgoal$ backwards, breaks down at each time step. At the final timestep T, we get:
\begin{align*}
  \qmdpt{T}(\staterobot,\actionrobot) &= \costgoal(\staterobot,\actionrobot)\\
  &= \costtarg(\staterobot, \actionrobot) \qquad \text{for any $\target$}\\
  \vmdpt{T}(\staterobot) &= \min_\actionrobot \costgoal(\staterobot, \actionrobot)\\
  &= \min_\actionrobot \min_\target \costtargstar(\staterobot, \actionrobot) \\
  &= \min_\target \vtargt{T}(\staterobot)
\end{align*}
Since $\vtargt{T}(\staterobot) = \min_\actionrobot \costtargstar(\staterobot, \actionrobot)$ by definition. Now, we show the recursive step:
\begin{align*}
  \qmdpt{t-1}(\staterobot,\actionrobot) &= \costgoal(\staterobot,\actionrobot) + \vmdpt{t}(\staterobot')\\
  &= \costtargstar(\staterobot,\actionrobot) + \min_\target \vtargt{t}(\staterobot') \hspace{1.3em} \target^* = \argmin \vtarg(\staterobot')\\
  &= \costtargstar(\staterobot,\actionrobot) + \vtargstart{t}(\staterobot') \qquad \target^* = \argmin \vtarg(\staterobot')\\
  \vmdpt{t-1}(\staterobot) &= \min_\actionrobot \qmdpt{t-1}(\staterobot, \actionrobot)\\
  &=  \min_\actionrobot \costtargstar(\staterobot,\actionrobot) + \vtargstart{t}(\staterobot') \hspace{1.0em} \target^* = \argmin \vtarg(\staterobot')\\
  & \geq  \min_\actionrobot \min_\target \left( \costtarg(\staterobot,\actionrobot) + \vtargt{t}(\staterobot') \right)\\
  &= \min_\target \vtargt{t-1}(\staterobot)
\end{align*}
Additionally, we know that $\vmdp(\staterobot) \leq \min_{\target} \vtarg(\staterobot)$, since $\vtarg(\staterobot)$ measures the cost-to-go for a specific target, and the total cost-to-go is bounded by this value for a deterministic system. Therefore, $\vmdp(\staterobot) = \min_{\target} \vtarg(\staterobot)$.
\end{proof}
%todo is this fine without using user action? seems like there should be something about that here \ldots

\subsection{Multi-Target Prediction}
Here, we don't assign the goal cost to be the cost of a single target $\costtarg$, but instead use a distribution over targets.%based on the cost-to-go.
\begin{theorem} \label{thm:mingoal_pred}
  Define the probability of a trajectory and target as $p(\traj, \target) \propto \exp(-\costtarg(\traj))$. Let $\vtargsoft$ and $\qtargsoft$ be the soft-value functions for target $\target$. The soft value functions for goal $\goal$, $\vgoalsoft$ and $\qgoalsoft$, can be computed as:
\begin{align*}
  \vgoalsoft(\staterobot) &= \softmin_\target \vtargsoft(\staterobot)\\
  \qgoalsoft(\staterobot, \actionuser) &= \softmin_\target \qtargsoft(\staterobot, \actionuser)
\end{align*}
\end{theorem}
\begin{proof}
  %For simplicity of exposition, we assume deterministic transition function $\transitionuser$. Unlike \thmref{thm:mingoal_assist}, this proof also holds for stochastic dynamics.
As the cost is additive along the trajectory, we can expand out $\exp(-\costtarg(\traj))$ and marginalize over future inputs to get the probability of an input now:
\begin{align*}
  \policyuser(\actionuser_t,\target| \staterobot_t) &= \frac{ \exp(-\costtarg(\staterobot_t, \actionuser_t)) \int \exp(-\costtarg(\trajatp{\staterobot_{t+1}})) } {\sum_{\target'}\int \exp(-\costtargprime(\trajat{\staterobot_{t}}))} 
\end{align*}
Where the integrals are over all trajectories. By definition, $\exp(-\vtargsoftt{t}(\staterobot_t)) = \int \exp(-\costtarg(\trajat{\staterobot_t}))$:
%$\exp(-\vtargsoftt{t} (\staterobot_t)) = \int \exp(-\costtarg(\trajat{\staterobot_t}))$:
\begin{align*}
  &= \frac{ \exp(-\costtarg(\staterobot_t, \actionuser_t)) \exp(-\vtargsoftt{t+1}(\staterobot_{t+1}))} {\sum_{\target'} \exp(-\vsoft_{\target',t}(\staterobot_{t}) )} \\
  %&= \frac{ \exp(-(\costtarg(\staterobot_t, \actionuser_t) + \vtargsoft(\staterobot_{t+1}))} {\sum_{\target'} \exp(-\vsoft_{\target'}(\staterobot_{t}) )} \\
  &= \frac{ \exp(-\qtargsoftt{t}(\staterobot_t, \actionuser_t))} {\sum_{\target'} \exp(-\vsoft_{\target',t}(\staterobot_{t}) )} 
\end{align*}
Marginalizing out $\target$ and simplifying:
\begin{align*}
  & \policyuser(\actionuser_t| \staterobot_t) = \frac{\sum_\target \exp( -\qtargsoftt{t}(\staterobot_t, \actionuser_t))} {\sum_{\target} \exp(-\vtargsoftt{t}(\staterobot_{t}) )} \\
  &= \exp \left( \log \left( \frac{\sum_\target \exp( -\qtargsoftt{t}(\staterobot_t, \actionuser_t))} {\sum_{\target} \exp(-\vtargsoftt{t}(\staterobot_{t}) )} \right) \right)\\
  %&= \exp \left( \log  \sum_\target \exp( -\qtargsoft(\staterobot_t, \actionuser_t))  - \log \sum_{\target} \exp(-\vtargsoft(\staterobot_{t}) )  \right)\\
  &= \exp \left( \softmin_\target \vtargsoftt{t}(\staterobot_t) - \softmin_\target \qtargsoft{t}(\staterobot_t, \actionuser_t) \right)
\end{align*}
As $\vgoalsoftt{t}$ and $\qgoalsoftt{t}$ are defined such that $\policyuser_t(\actionuser | \staterobot, \goal) = \exp(\vgoalsoftt{t}(\staterobot) -\qgoalsoftt{t}(\staterobot, \actionuser))$, our proof is complete.
\end{proof}

%
%Marginalizing out g:
%\begin{align*}
%  p(a_t | s) &= \sum_g p(a_t, g | s)\\
%  &= \frac{ \sum_g \exp(-Q_g^{t}(s_t, a_t))} {\sum_{g'}\exp(-V_{g'}^{t}(s_{t}))}
%\end{align*}
%
%We can also write this out as:
%\begin{align*}
%  \exp\left( \log\left( p(a_t | s) \right) \right)&= \exp\left( \log\left(  \frac{ \sum_g \exp(-Q_g^{t}(s_t, a_t))} {\sum_{g'}\exp(-V_{g'}^{t}(s_{t}))}\right) \right)\\
%  &= \exp\left( \log\left(  \sum_g \exp(-Q_g^{t}(s_t, a_t)) \right) - \log\left(\sum_{g'}\exp(-V_{g'}^{t}(s_{t})) \right) \right)\\
%  &= \exp\left( \softmin_g V_{g}^{t}(s_{t}) - \softmin_g Q_g^{t}(s_t, a_t)\right)
%\end{align*}
%

\begin{figure}[t]
\centering
 \begin{subfigure}{0.24\textwidth}
   \centering 
   \includegraphics[width=0.97\textwidth, trim=440 250 500 210, clip=true]{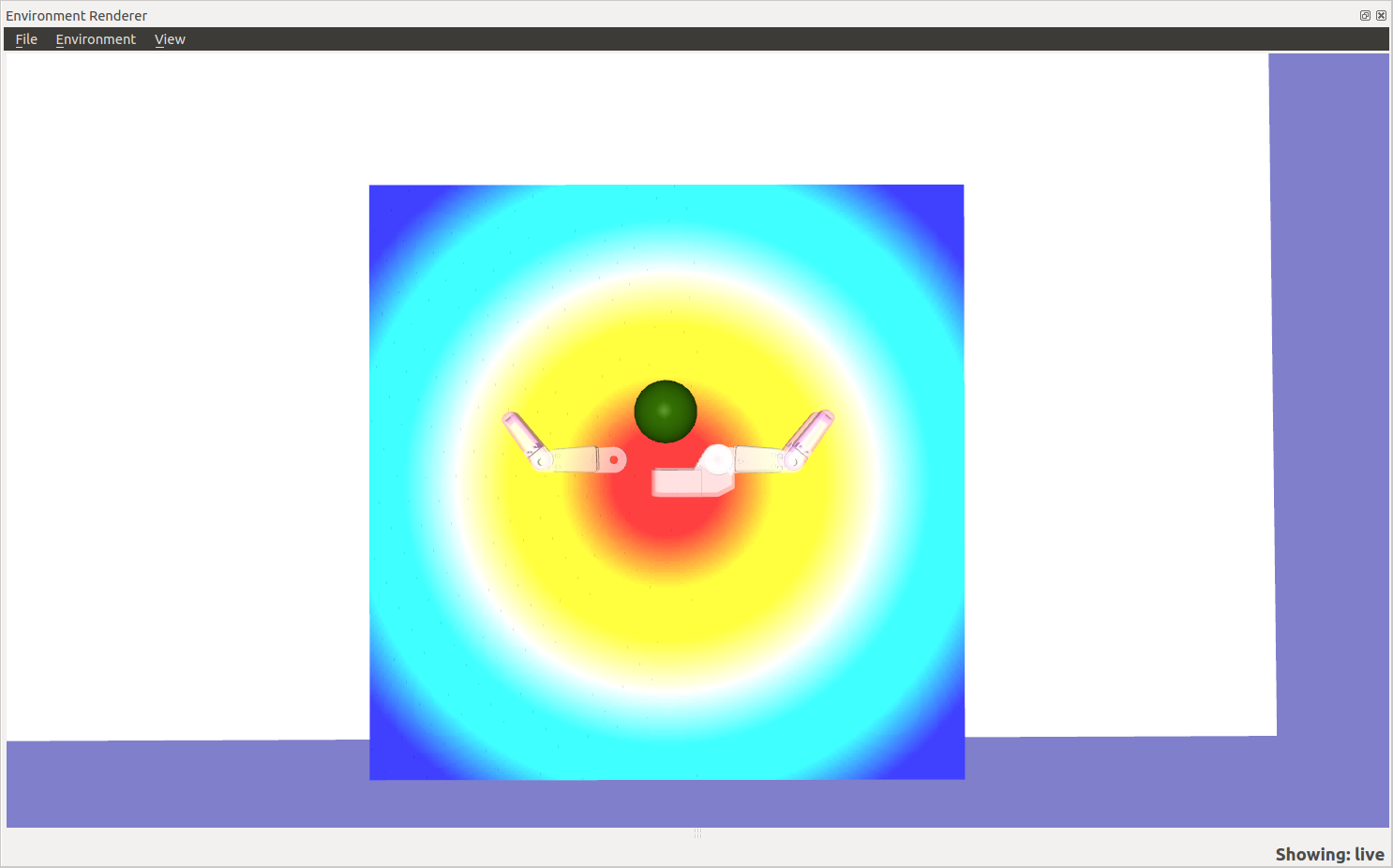}
  \caption{}
 \label{fig:multigoal_1}
 \end{subfigure}
 \begin{subfigure}{0.24\textwidth}
   \centering 
   \includegraphics[width=0.97\textwidth, trim=440 250 500 210, clip=true]{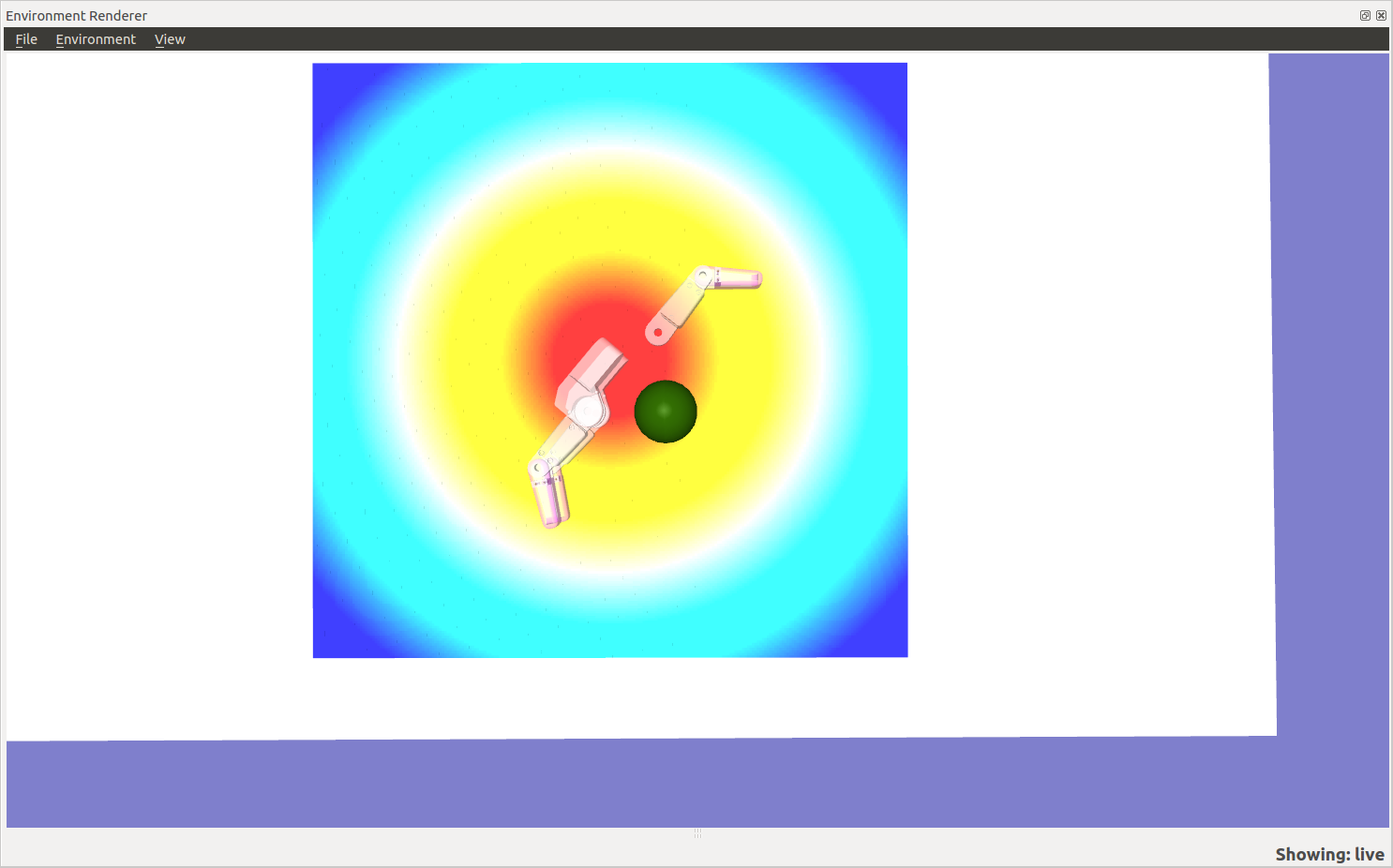}
  \caption{}
 \label{fig:multigoal_2}
 \end{subfigure}
 \begin{subfigure}{0.24\textwidth}
   \centering 
   \includegraphics[width=0.97\textwidth, trim=440 250 500 210, clip=true]{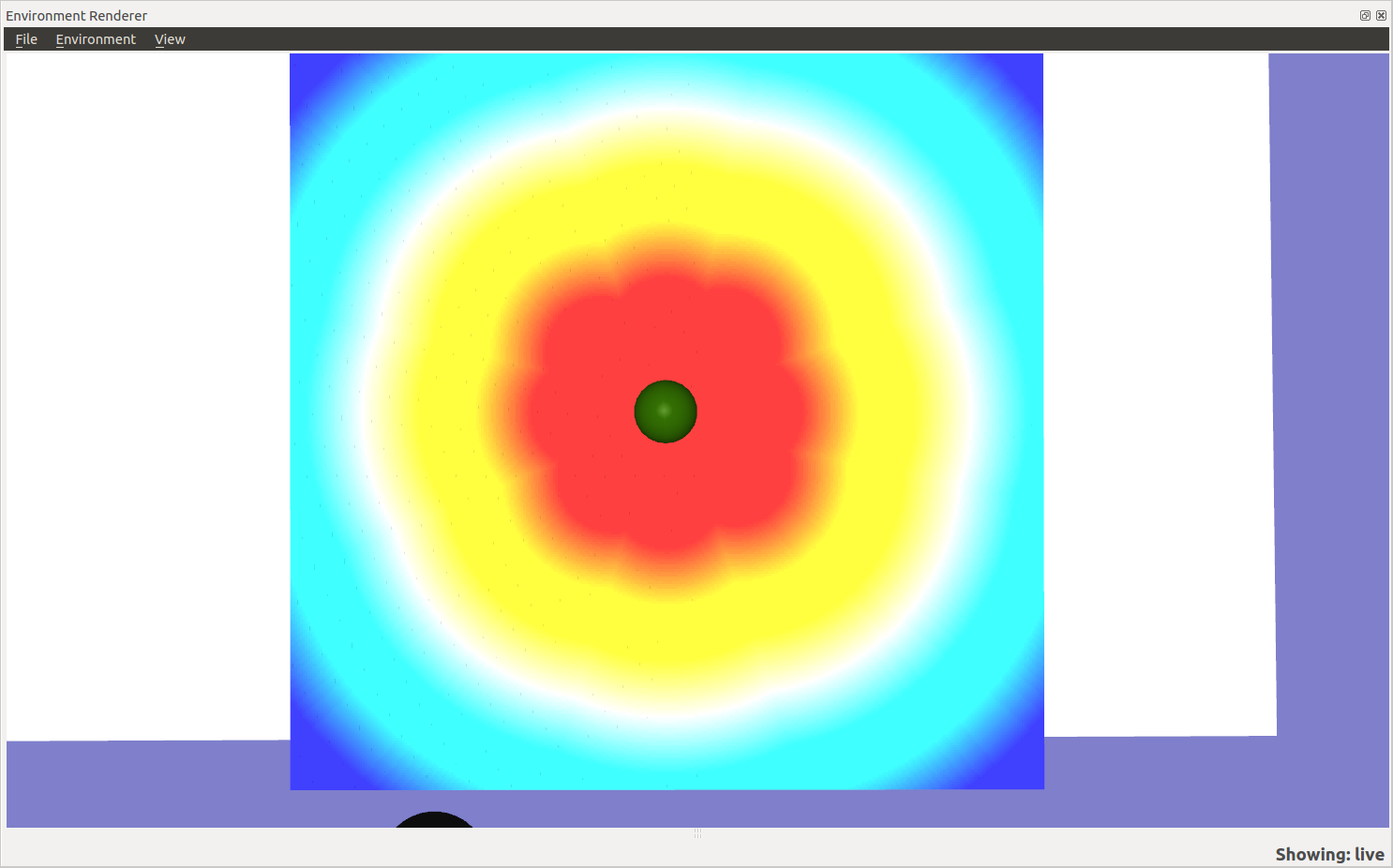}
  \caption{}
 \label{fig:multigoal_3_arb}
 \end{subfigure}
 \begin{subfigure}{0.24\textwidth}
   \centering 
   \includegraphics[width=0.97\textwidth, trim=440 250 500 210, clip=true]{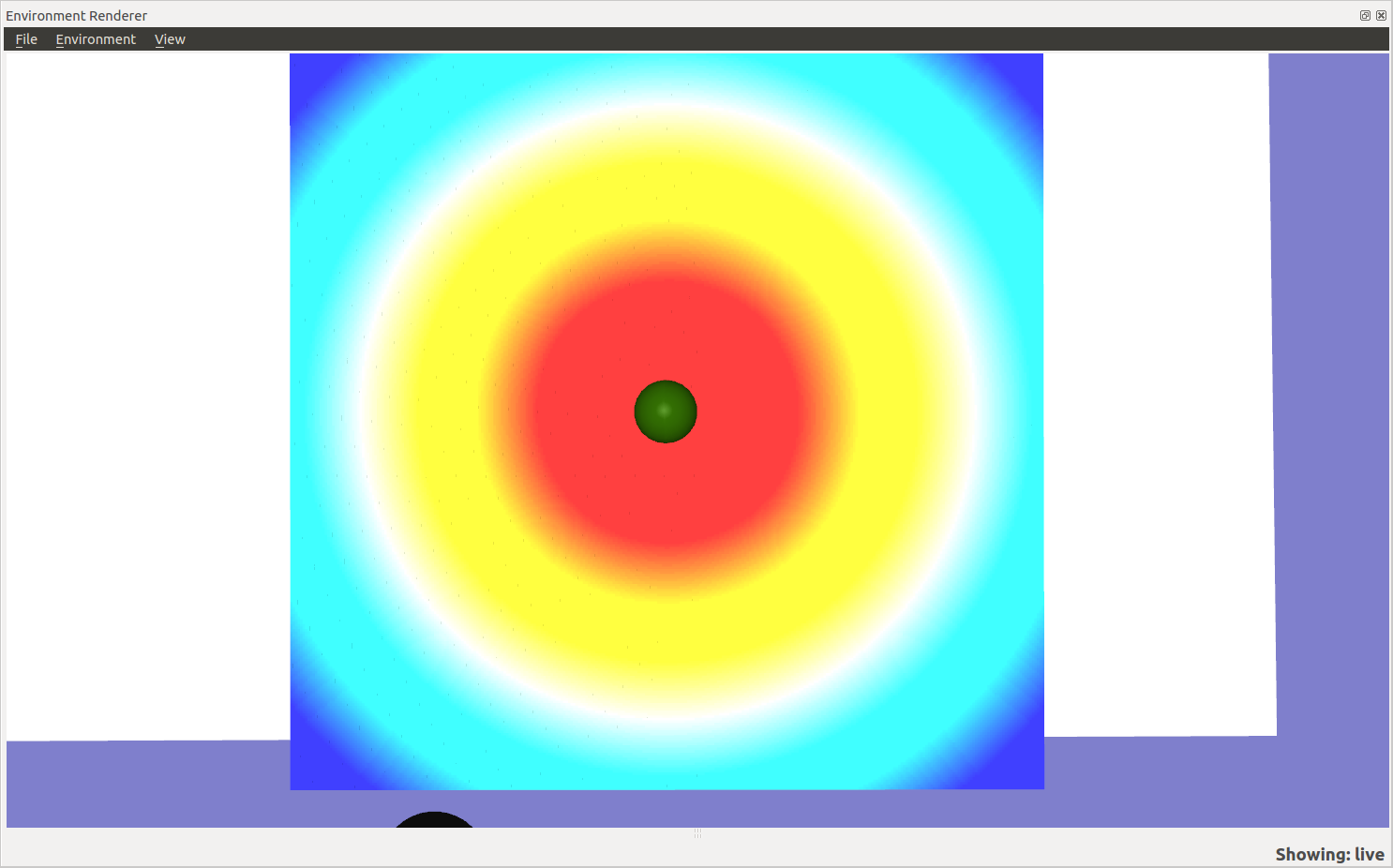}
  \caption{}
 \label{fig:multigoal_3_pred}
 \end{subfigure}
 \caption{Value function for a goal (grasp the ball) decomposed into value functions of targets (grasp poses). (\subref{fig:multigoal_1}, \subref{fig:multigoal_2}) Two targets and their corresponding value function $\vtarg$. In this example, there are 16 targets for the goal. (\subref{fig:multigoal_3_arb}) The value function of a goal $\vgoal$ used for assistance, corresponding to the minimum of all 16 target value functions (\subref{fig:multigoal_3_pred}) The soft-min value function $\vgoalsoft$ used for prediction, corresponding to the soft-min of all 16 target value functions.}
 \label{fig:multigoal}
\end{figure}

\section{User Study} 
\label{sec:experiments}
We compare two methods for shared autonomy in a user study: our method, referred to as \emph{policy}, and a conventional predict-then-blend approach based on Dragan and Srinivasa~\cite{dragan_2013_assistive}, referred to as \emph{blend}.
 
Both systems use the same prediction algorithm, based on the formulation described in \sref{sec:prediction}. For computational efficiency, we follow Dragan and Srinivasa~\cite{dragan_2013_assistive} and use a second order approximation about the optimal trajectory. They show that, assuming a constant Hessian, we can replace the difficult to compute soft-min functions $\vtargsoft$ and $\qtargsoft$ with the min value and action-value functions $\vtarg$ and $\qtarg$. 
 
Our policy approach requires specifying two cost functions, $\costtarguser$ and $\costtargrobot$, from which everything is derived. For $\costtarguser$, we use a simple function based on the distance $d$ between the robot state $\staterobot$ and target $\target$:
 \begin{align*}
   \costtarguser(\staterobot, \actionuser) &= \left\{ \begin{array}{cc} \alpha & d > \delta \\ \frac{\alpha}{\delta} d & d\leq \delta \end{array} \right.
 \end{align*}

That is, a linear cost near a goal $(d \leq \delta)$, and a constant cost otherwise. This is by no means the best cost function, but it does provide a baseline for performance. We might expect, for example, that incorporating collision avoidance into our cost function may enable better performance~\cite{hauser_2011}.
 
We set $\costtargrobot(\staterobot, \actionrobot, \actionuser) = \costtarguser(\staterobot, \actionuser) + (\actionrobot - \userinputtoaction(\actionuser))^2$, penalizing the robot from deviating from the user command while optimizing their cost function. %As described, the rest of the necessary quantities follow.

The predict-then-blend approach of Dragan and Srinivasa requires estimating how confident the predictor is in selecting the most probable goal. This confidence measure controls how autonomy and user input are arbitrated. For this, we use the distance-based measure used in the experiments of Dragan and Srinivasa~\cite{dragan_2013_assistive}, $\text{conf} = \max\left(0, 1-\frac{d}{D}\right)$, where $d$ is the distance to the nearest target, and $D$ is some threshhold past which confidence is zero.

\subsection{Hypotheses}
\label{sec:hypoths}
Our experiments aim to evaluate the task-completion efficiency and user satisfaction of our system compared to the predict-then-blend approach. Efficiency of the system is measured in two ways: the total execution time, a common measure of efficiency in shared teleoperation~\cite{crandall_2002}, and the total user input, a measure of user effort. User satisfaction is assessed through a survey.

This leads to the following hypotheses:

\noindent \textbf{H1} \indent \textit{Participants using the policy method will grasp objects significantly faster than the blend method}

\noindent \textbf{H2} \indent \textit{Participants using the policy method will grasp objects with significantly less control input than the blend method}

\noindent \textbf{H3} \indent \textit{Participants will agree more strongly on their preference for the policy method compared to the blend method}

\subsection{Experiment setup}
\label{sec:experiment_setup}

\begin{figure}[t]
\centering
\includegraphics[width=0.49\textwidth, trim=150 0 200 0, clip=true]{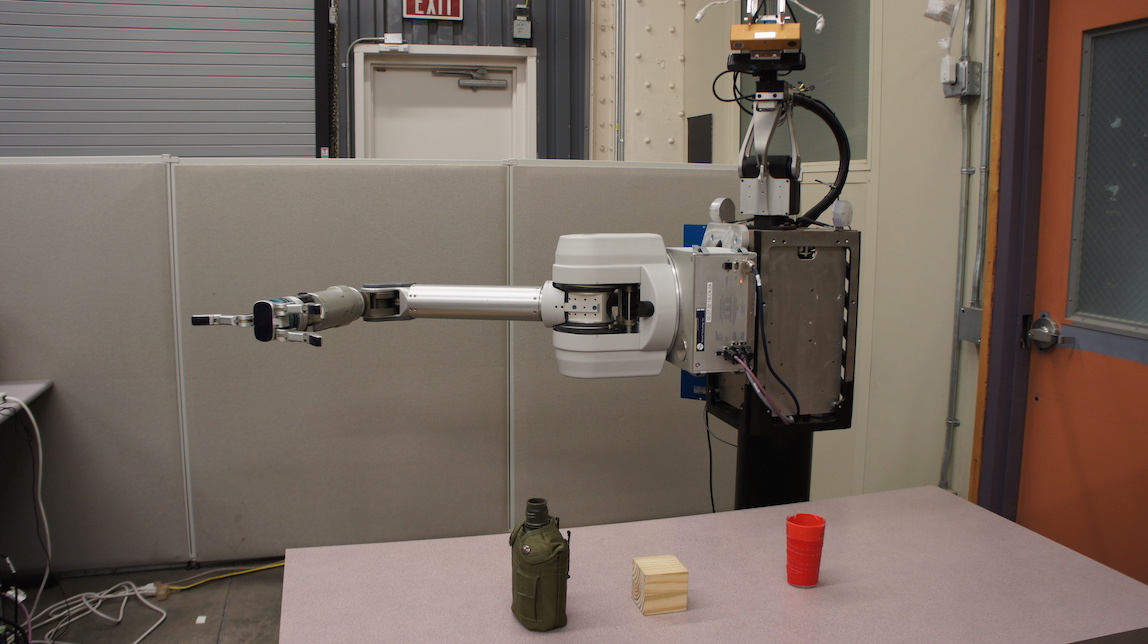}
\caption{Our experimental setup for object grasping. Three objects - a canteen, block, and glass - were placed on the table in front of the robot in a random order. Prior to each trial, the robot moved to the configuration shown. Users picked up each object using each teleoperation system.}%For each trail, the user was told which of the three objects to pickup.}%They teleoperated the arm to a grasp pose, then pushed a button to close the hand.}
 \label{fig:exper_setup}
\end{figure} 

We recruited 10 participants (9 male, 1 female), all with experience in robotics, but none with prior exposure to our system. To counterbalance individual differences of users, we chose a within-subjects design, where each user used both systems.

%change to horizontal plane and height?
We setup our experiments with three objects on a table - a canteen, a block, and a cup. See \figref{fig:exper_setup}. Users teleoperated a robot arm using two joysticks on a Razer Hydra system. The right joystick mapped to the horizontal plane, and the left joystick mapped to the height. A button on the right joystick closed the hand. Each trial consisted of moving from the fixed start pose, shown in \figref{fig:exper_setup}, to the target object, and ended once the hand was closed.

%TODO
%Trained users, but did not provide algorithmic details of the method (algorithms and system differences never explained)
%users went through a training procedure, then put semicolon, then explain each item of the procedure
At the start of the study, users were told they would be using two different teleoperation systems, referred to as ``method1'' and ``method2''. Users were not provided any information about the methods. Prior to the recorded trials, users went through a training procedure: First, they teleoperated the arm directly, without any assistance or objects in the scene. Second, they grasped each object one time with each system, repeating if they failed the grasp. Third, they were given the option of additional training trials for either system if they wished.

Users then proceeded to the recorded trials. For each system, users picked up each object one time in a random order. Half of the users did all blend trials first, and half did all policy trials first. Users were told they would complete all trials for one system before the system switched, but were not told the order. However, it was obvious immediately after the first trail started, as the policy method assists from the start pose and blend does not. Upon completing all trials for one system, they were told the system would be switching, and then proceeded to complete all trials for the other system. If users failed at grasping (e.g. they knocked the object over), the data was discarded and they repeated that trial. Execution time and total user input were measured for each trial.

%Users then proceeded to the six recorded trials, one for each method and object. We randomized the order of assistance systems, and the order which users grasped each object (but the same order for both systems). If users failed at grasping (e.g. they knocked the object over), the data was discarded and they repeated that trial. Execution time and total user input were measured for each trial.
%This gives us a total of six data points, three for each method. However, as these are not independent data points, we average the time and control input for all three trials, giving us two data points per user, one for each system.

Upon completing all trials, users were given a short survey. For each system, they were asked for their agreement on a 1-7 Likert scale for the following statements:
\begin{enumerate}
  \item ``I felt in \emph{control}''
  \item ``The robot did what I \emph{wanted}''
  \item ``I was able to accomplish the tasks \emph{quickly}''
  \item ``If I was going to teleoperate a robotic arm, I would \emph{like} to use the system''
\end{enumerate}

They were also asked ``which system do you \emph{prefer}'', where $1$ corresponded to blend, $7$ to policy, and $4$ to neutral. Finally, they were asked to explain their choices and provide any general comments. %User averages for all questions are shown in \figref{fig:survey_means}.

\subsection{Results}

\begin{figure}[t]
 % \begin{subfigure}{0.1028\textwidth}
   \includegraphics{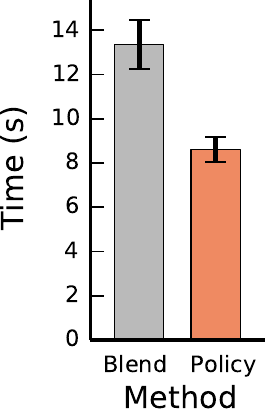}
   %\caption{Mean}
 %\label{fig:time_compare}
 %\end{subfigure}
 \hfill
 %\begin{subfigure}{0.34\textwidth}
 %  \centering 
   %\includegraphics[trim=0 0 0 0, clip=true]{images/time_diff_error.pdf}
   \includegraphics[trim=0 0 0 0, clip=true]{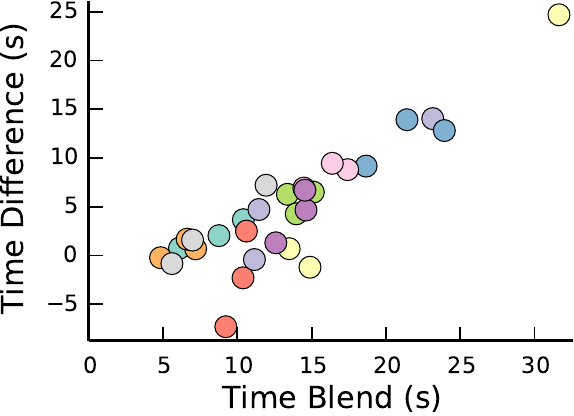}
  %\caption{Per-Trial Time Difference}
 %\label{fig:time_diffs}
 %\end{subfigure}

   \includegraphics{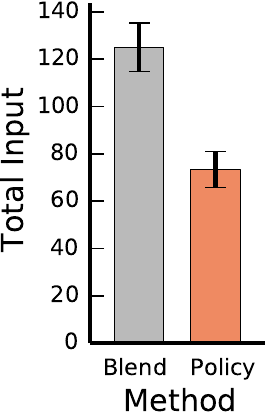}
   \hfill
   \includegraphics[trim=0 0 0 0, clip=true]{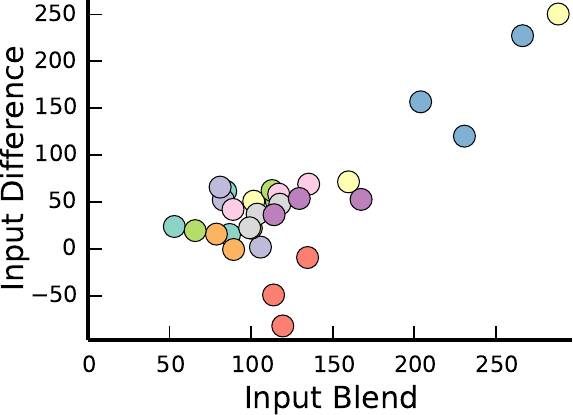}
   \caption{Task completion times and total input for all trials. On the left, means and standard errors for each system. On the right, the time and input of blend minus policy, as a function of the time and total input of blend. Each point corresponds to one trial, and colors correspond to different users. We see that policy was faster and resulted in less input in most trials. Additionally, the difference between systems increases with the time/input of blend.}
 \label{fig:time_control_plots}
 \end{figure}

% \begin{figure}
%  %\begin{subfigure}[b]{0.1028\textwidth}
%   \includegraphics{images/controlplot.pdf}
%   %\caption{Mean}
% %\label{fig:control_compare}
% %\end{subfigure}
% \hfill
% %\begin{subfigure}[b]{0.34\textwidth}
% %  \centering 
%   \includegraphics[trim=0 0 0 0, clip=true]{images/control_diff.pdf}
%  %\caption{Per-Trial Time Difference}
% %\label{fig:control_diffs}
% %\end{subfigure}
% \caption{Task completion times and total control applied by users for each policy. Means and standard errors are plotted.}
% \label{fig:controlplots}
% \end{figure}

\begin{figure}[t]
   \centering 
   %\begin{subfigure}{0.27\textwidth}
   %  \centering 
     \includegraphics[trim=0 0 0 0, clip=true]{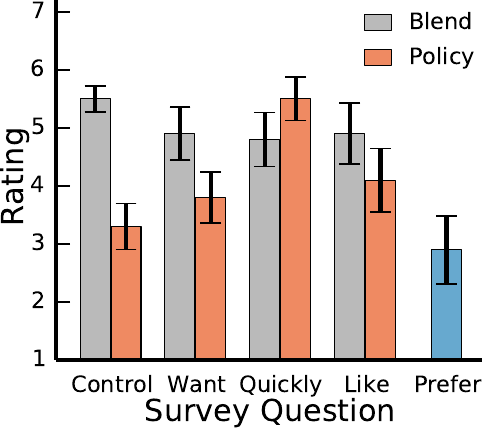}
   %   \caption{ }
   %  \label{subfig:survey_means}
   %\end{subfigure}
   \hfill
   %\begin{subfigure}{0.21\textwidth}
   %  \centering 
     \includegraphics[trim=0 0 0 0, clip=true]{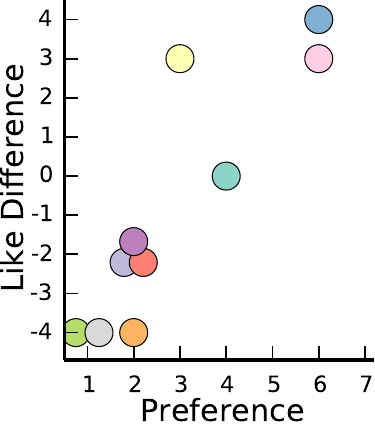}
   %   \caption{ }
   %   \label{subfig:prev_vs_like}
   %\end{subfigure}
   %\includegraphics[trim=0 0 0 0, clip=true]{images/survey_plot_means.pdf}
   \caption{On the left, means and standard errors from survey results from our user study. For each system, users were asked if they felt in \emph{control}, if the robot did what they \emph{wanted}, if they were able to accomplish tasks \emph{quickly}, and if they would \emph{like} to use the system. Additionally, they were asked which system they \emph{prefer}, where a rating of 1 corresponds to blend, and 7 corresponds to policy. On the right, the \emph{like} rating of policy minus blend, plotted against the \emph{prefer} rating. When multiple users mapped to the same coordinate, we plot multiple dots around that coordinate. Colors correspond to different users, where the same user has the same color in \figref{fig:time_control_plots}. }
 \label{fig:survey_means}
 \end{figure}

Users were able to successfully use both systems. There were a total of two failures while using each system - once each because the user attempted to grasp too early, and once each because the user knocked the object over. These experiments were reset and repeated.

We assess our hypotheses using a significance level of $\alpha=0.05$, and the Benjamini–Hochberg procedure to control the false discovery rate with multiple hypotheses.

Trial times and total control input were assessed using a two-factor repeated measures ANOVA, using the assistance method and object grasped as factors. Both trial times and total control input had a significant main effect. We found that our policy method resulted in users accomplishing tasks more quickly, supporting \textbf{H1} $(F(1,9)=12.98, p=0.006)$. Similarly, our policy method resulted in users grasping objects with less input, supporting \textbf{H2} $(F(1,9)=7.76, p = 0.021)$. See \figref{fig:time_control_plots} for more detailed results.

%blend took on average $55\%$ longer than policy. 13.37s to 8.60s
%blend resulted in users applying $70\%$ more control on average. 124.92 to 73.33
%No significant effect was found on the trial run, and no interaction effects were found.

To assess user preference, we performed a Wilcoxon paired signed-rank test on the survey question asking if they would \emph{like} to use each system, and a Wilcoxon rank-sum test on the survey question of which system they \emph{prefer} against the null hypothesis of no preference (value of 4). There was no evidence to support \textbf{H3}.

In fact, our data suggests a trend towards the opposite - that users prefer blend over policy. When asked if they would \emph{like} to use the system, there was a small difference between methods (Blend: $M=4.90, SD=1.58$, Policy: $M=4.10, SD=1.64)$. However, when asked which system they \emph{preferred}, users expressed a stronger preference for blend ($M=2.90, SD=1.76$). While these results are not statistically significant according to our Wilcoxon tests and $\alpha=0.05$, it does suggest a trend towards preferring blend. See \figref{fig:survey_means} for results for all survey questions.

%We found no evidence to support \textbf{H3}, where a Wilcoxon paired signed-rank test showed no significance. In fact, our data suggests a trend toward the opposite - that users prefer blend over policy. When asked if they would \emph{like} to use the system, there was a small difference between methods. However, when asked which system they preferred directly, users expressed a stronger preference for blend with a mean of $2.90$ on a 1-7 Likert scale. See \figref{fig:survey_means} for results for all survey questions.

We found this surprising, as prior work indicates a strong correlation between task completion time and user satisfaction, even at the cost of control authority, in both shared autonomy~\cite{dragan_2013_assistive, hauser_2013} and human-robot teaming~\cite{gombolay_2014} settings.\footnote{In prior works where users preferred greater control authority, task completion times were indistinguishable~\cite{kim_2012}.} Not only were users faster, but they recognized they could accomplish tasks more quickly (see \emph{quickly} in \figref{fig:survey_means}). One user specifically commented that  ``(Policy) took more practice to learn\ldots but once I learned I was able to do things a little faster. However, I still don't like feeling it has a mind of it's own''.

%other data: 
%wilcox control: (0.0, 0.0072125783459572193)
%wilcox quickly: (5.0, 0.23555891992407951)
%wilcox want: (5.0, 0.1247727360437228)
%wilcox tedious: (6.0, 0.68309139830960874)
%wilcox rating (16.5, 0.47130320504429746)
%wilcox compare(9.0, 0.098458502377584176)

As shown in \figref{fig:survey_means}, users agreed more strongly that they felt in \emph{control} during blend. Interestingly, when asked if the robot did what they \emph{wanted}, the difference between methods was less drastic. This suggests that for some users, the robot's autonomous actions were in-line with their desired motions, even though the user was not in control.

Users also commented that they had to compensate for policy in their inputs. For example, one user stated that ``(policy) did things that I was not expecting and resulted in unplanned motion''. This can perhaps be alleviated with user-specific policies, matching the behavior of particular users.

%Fundamentally, we seem to treat a user as a provider of goals, using their actions to infer intent an assist. However, the interface asks them to execute trajectories

Some users suggested their preferences may change with better understanding. For example, one user stated they ``disliked (policy) at first, but began to prefer it slightly after learning its behavior. Perhaps I would prefer it more strongly with more experience''. It is possible that with more training, or an explanation of how policy works, users would have preferred the policy method. We leave this for future work.

\subsection{Examining trajectories}

%We also look at traces of trajectories to determine why some users disliked the autonomous policy.

\begin{figure}[t]
  \centering
  \begin{subfigure}{0.233\textwidth}
    \centering 
    \includegraphics[trim=0 0 0 0, clip=true]{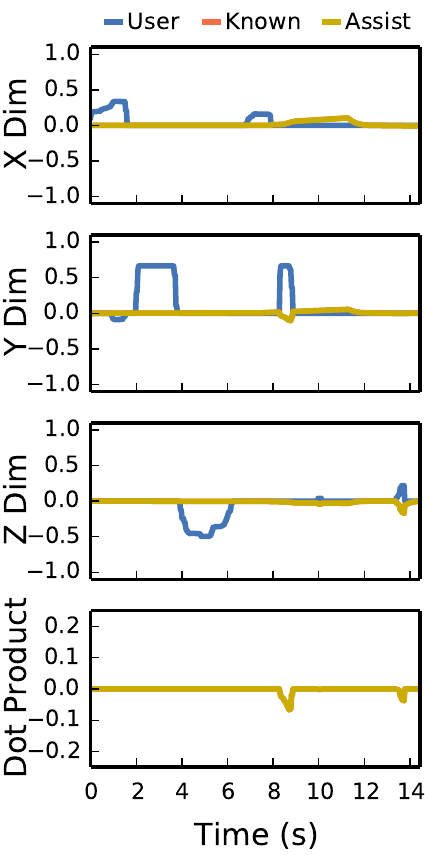}
    \caption{Blend}
    \label{fig:user7_blend}
  \end{subfigure}
  \hfill
  \begin{subfigure}{0.233\textwidth}
    \centering 
    \includegraphics[trim=0 0 0 0, clip=true]{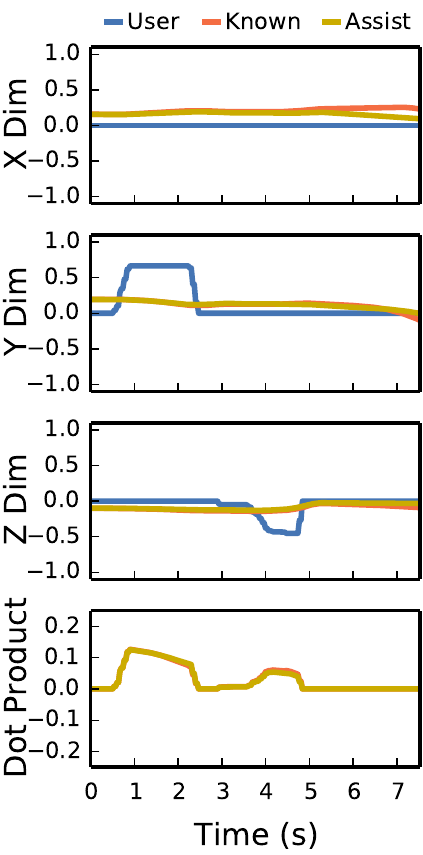}
    \caption{Policy}
    \label{fig:user7_policy}
  \end{subfigure}
  \caption{User input and autonomous actions for a user who preferred policy assistance, using (\subref{fig:user7_blend})~blending and (\subref{fig:user7_policy})~policy for grasping the same object. We plot the user input, autonomous assistance with the estimated distribution, and what the autonomous assistance would have been had the predictor known the true goal. We subtract the user input from the assistance when plotting, to show the autonomous action as compared to direct teleoperation. The top 3 figures show each dimension separately. The bottom shows the dot product between the user input and assistance action. This user changed their strategy during policy assistance, letting the robot do the bulk of the work, and only applying enough input to correct the robot for their goal. Note that this user never applied input in the `X' dimension in this or any of their three policy trials, as the assistance always went towards all objects in that dimension.}
  \label{fig:user7}
\end{figure}

\begin{figure}[t]
  \centering
  \begin{subfigure}{0.233\textwidth}
    \centering 
    \includegraphics[trim=0 0 0 0, clip=true]{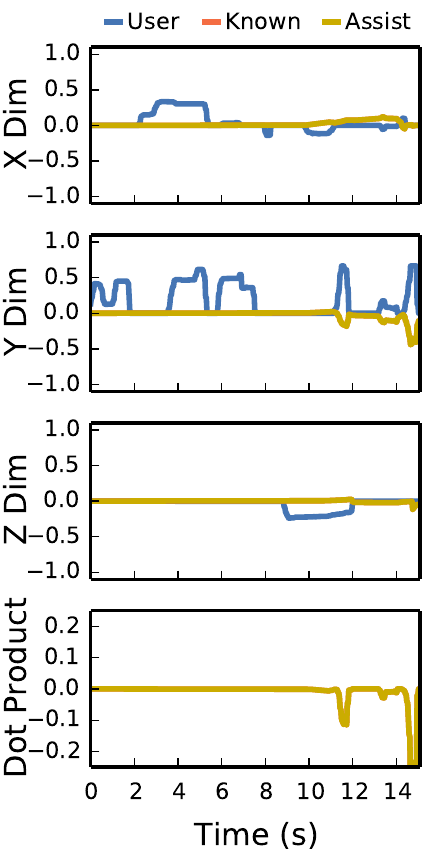}
    \caption{Blend}
    \label{fig:user6_blend}
  \end{subfigure}
  \hfill
  \begin{subfigure}{0.233\textwidth}
    \centering 
    \includegraphics[trim=0 0 0 0, clip=true]{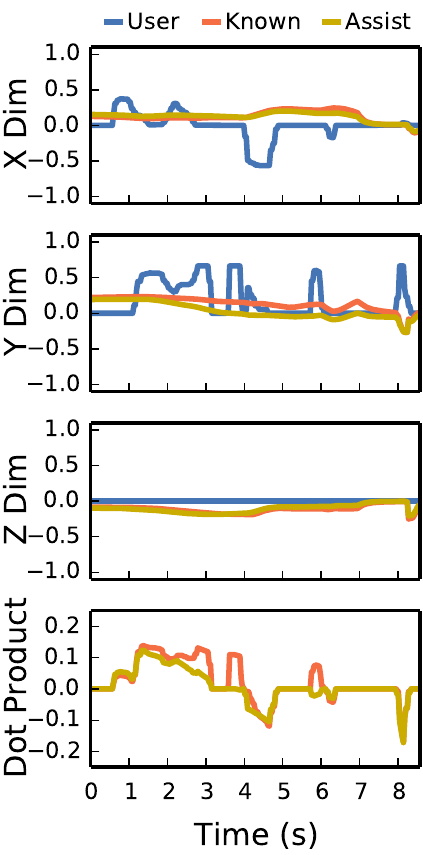}
    \caption{Policy}
    \label{fig:user6_policy}
  \end{subfigure}
  \caption{User input and autonomous assistance for a user who preferred blending, with plots as in \figref{fig:user7}. The user inputs sometimes opposed the autonomous assistance (such as in the `X' dimension) for both the estimated distribution and known goal, suggesting the cost function didn't accomplish the task in the way the user wanted. Even still, the user was able to accomplish the task faster with the autonomous assistance then blending. }
  \label{fig:user6}
\end{figure}

Users with different preferences had very different strategies for using each system. Some users who preferred the assistance policy changed their strategy to take advantage of the constant assistance towards all goals, applying minimal input to guide the robot to the correct goal (\figref{fig:user7}). In contrast, users who preferred blending were often opposing the actions of the autonomous policy (\figref{fig:user6}). This suggests the robot was following a strategy different from their own.

%Comments:
%-People were startled during practice when robot started moving on its own. Some went through multiple trials of doing nothing until they interacted with the robot
%
%-Some users tended to ``fight'' the system. Other's worked with it.
%---Maybe show the user input trace, one for each?

\section{Conclusion and Future Work} 
\label{sec:conclusion}

We presented a framework for formulating shared autonomy as a POMDP. Whereas most methods in shared autonomy predict a single goal, then assist for that goal (predict-then-blend), our method assists for the entire distribution of goals, enabling more efficient assistance. We utilized the MaxEnt IOC framework to infer a distribution over goals, and Hindsight Optimization to select assistance actions. We performed a user study to compare our method to a predict-then-blend approach, and found that our system enabled faster task completion with less control input. Despite this, users were mixed in their preference, trending towards preferring the simpler predict-then-blend approach.

We found this surprising, as prior work has indicated that users are willing to give up control authority for increased efficiency in both shared autonomy~\cite{dragan_2013_assistive, hauser_2013} and human-robot teaming~\cite{gombolay_2014} settings. Given this discrepancy, we believe more detailed studies are needed to understand precisely what is causing user dissatisfaction. Our cost function could then be modified to explicitly avoid dissatisfying behavior. Additionally, our study indicates that users with different preferences interact with the system in very different ways. This suggests a need for personalized learning of cost functions for assistance.

%TODO if add something showing user variance, talk about it here.
%We believe many potential sources of this frustration can can be incorporated into our framework as-is, altering the robot cost to explicitly penalize these frustrations.% e.g. ``user-fight'', the difference in direction between the user input and robot action. 

Implicit in our model is the assumption that users do not consider assistance when providing inputs - and in particular, that they do not adapt their strategy to the assistance. We hope to alleviate this assumption in both prediction and assistance by extending our model as a stochastic game.
%Methods that consider some explicit information gathering, while working in continuous domains \cite{hauser_2010_rbsr}

\section*{Acknowledgments}
This work was supported in part by NSF GRFP No. DGE-1252522, NSF Grant No. 1227495, the DARPA Autonomous Robotic Manipulation Software Track program, the Okawa Foundation, and an Office of Naval Research Young Investigator Award.

%Drew
%NSF Grant No. 1227495 %purposeful prediction
% or cite like this
% National Science Foundation (Purposeful Prediction: Co-robot Interaction via Understanding Intent and Goals)
%Okawa Foundation
%DARPA Autonomous Robotic Manipulation Software Track Program

%Sidd
%Office of Naval Research Young Investigator Award

%\clearpage

%% Use plainnat to work nicely with natbib. 

\bibliographystyle{plainnat}
\bibliography{references}

\end{document}